\documentclass{article}

\usepackage[preprint]{neurips_2026}

\usepackage[utf8]{inputenc}
\usepackage[T1]{fontenc}
\usepackage{microtype}
\usepackage{graphicx}
\usepackage{booktabs} 
\usepackage{multirow}
\usepackage{float}
\graphicspath{{./figures/}}

\usepackage{hyperref}
\usepackage{url}
\usepackage{xcolor}

\usepackage{amsmath}
\usepackage{amssymb}
\usepackage{mathtools}
\usepackage{amsthm}

\usepackage[capitalize,noabbrev]{cleveref}

\theoremstyle{plain}
\newtheorem{theorem}{Theorem}[section]
\newtheorem{proposition}[theorem]{Proposition}

\theoremstyle{definition}

\theoremstyle{remark}
\newtheorem{remark}[theorem]{Remark}

\usepackage[textsize=tiny]{todonotes}

\title{Martingale Neural Operators: Learning Stochastic Marginals via Doob-Meyer Factorization}

\author{%
	Kai Hidajat \\
	Department of Applied Mathematics \\
	University of Washington \\
	Seattle, WA, USA \\
	\texttt{hidajatk@uw.edu} \\
}

\begin{document}

\maketitle

\begin{abstract}
	Neural operators excel as deterministic surrogates, but inevitably collapse to the conditional mean when applied to stochastic PDEs, discarding the variance and tail structure upon which uncertainty quantification depends. Recovering this structure typically requires Monte Carlo rollouts or grafted generative models, both of which surrender the one-shot efficiency and resolution invariance that define the operator paradigm. To resolve this, we draw on the Doob-Meyer theorem, which establishes that any semimartingale fundamentally decomposes into a predictable drift and an unpredictable, zero-mean martingale. Translating this theorem into an architectural prior, we introduce the Martingale Neural Operator (MNO). MNO maps an initial condition directly to the conditional mean and covariance of the terminal law, parameterized by a drift-like mean and a low-rank factor \(B_\phi\) with \(B_\phi^\top B_\phi\) positive semi-definite by construction. For our experiments, we use a Gaussian residual instantiation. Across 1D SPDEs, rough volatility, and 2D operator tasks, MNO reduces Wasserstein distance by up to \(120\times\) on \(\phi^4\) field theory and \(68\times\) on stochastic Burgers, evaluating \(\sim 3\times\) faster than a conditional diffusion baseline at matched wall-clock training budgets. On 2D tasks, MNO is comparable to FNO on zero-shot resolution transfer and turbulent flow, while quasi-deterministic systems such as Gray-Scott remain a failure mode.
\end{abstract}

\section{Introduction}
\label{sec:intro}

Operator learning methods such as Fourier Neural Operator (FNO) \citep{Li_Kovachki_etal_2021} and DeepONet \citep{Lu_Jin_etal_2021} learn maps between infinite-dimensional function spaces, allowing a single trained model to be resolution invariant. Across discretizations, neural operators evaluate zero-shot, granting and orders of magnitude speedups over iterative classical solvers \citep{Kovachki_Li_etal_2024,Kossaifi_Kovachki_etal_2023}. However, most successes remain deterministic \citep{Salvi_Lemercier_etal_2022}. For stochastic systems, SPDE solutions are laws over fields or paths, and standard \(L^2\)-trained operators converge to the conditional mean \(\mathbb{E}[u_T\mid u_0]\). The conditional mean is the right regression estimator, but it erases the aleatoric residual. Since spatial variance \citep{Shi_Gao_etal_2017, Zhang_Fiore_etal_2021}, correlation \citep{Nordman_Lahiri_2004, Popovic_VidalCalleja_etal_2017}, and sampling variability \citep{Gofer_Praisler_etal_2021} are important for many tasks \citep{Sabbar_Nisar_2025}, this tradeoff prohibits many tasks from observing the speedup MSE neural operators offer.

Classical stochastic solvers and uncertainty-quantification (UQ) methods successfully retain this distributional information, incur additional costs \citep{Psaros_Meng_etal_2023}. Monte Carlo rollouts are robust but slow, while polynomial and Wiener-chaos expansions \citep{Xiu_Karniadakis_2006, Schobi_Sudret_etal_2015} can be sharp when the stochastic dimension and basis stay controlled, yet delicate for high-dimensional or rough drivers \citep{Higham_2001,Gatheral_Jaisson_etal_2014,Shi_Lin_etal_2026}. Generative approaches, including diffusion or flow samplers, can represent richer laws, yet often require iterative sampling or task-specific guidance when physical consistency is imposed \citep{Song_SohlDickstein_etal_2021,Yu_Li_etal_2026}. Many applications need the more specific object of terminal marginal moments or risk envelopes \citep{Acciaio_Flores_etal_2026, Zhao_Jia_2020, Zhao_Yin_2026, Majumdar_Pavone_2020}, and should have the benefit of operator speed.

We introduce Martingale Neural Operator (MNO), which targets exactly that object. In semimartingale settings, the Doob-Meyer/canonical decomposition already gives the right grammar to integrate regression estimation with UQ, namely finite variation plus martingale residual noise \citep{Doob_1937,Meyer_1962,Protter_2005}. For MNO, we adapt this split into an architectural prior for a neural operator which predicts terminal marginals. In non-semimartingale settings such as rough volatility, the same split is a principled mean/covariance factorization. The implemented model has FNO heads for both components and a rank-\(r\), channel-aware residual factor whose squared norm gives per-channel variance and whose product defines a PSD covariance.

Our contributions are as follows:
\begin{enumerate}
	\item \textbf{Drift/Covariance Structural Bias.} We propose a neural operator architecture that explicitly separates a predictable mean operator from a zero-mean stochastic residual covariance operator, motivated by the Doob-Meyer decomposition.
	\item \textbf{One-Shot Uncertainty Quantification.} MNO provides closed-form mean and per-channel variance fields, together with a low-rank covariance factor for sampling, enabling moment-based risk metrics without iterative rollout or sample averaging at inference time.
	\item \textbf{Benchmark Suite.} We evaluate MNO on benchmarks spanning 1D stochastic surrogates, rough volatility dynamics, and 2D operator tasks. MNO achieves up to \(120\times\) reduction in \(W_2\) over stochastic-surrogate baselines on 1D SPDE tasks, \(2.6\times\) improvement over sequential baselines on terminal-marginal accuracy in the deeply rough volatility regime (\(H=0.1\)), and \({\sim}3\times\) faster inference than a conditional diffusion model under a matched wall-clock training budget. The 2D results with our unoptimized experimental setup show MNO is competitive with FNO on turbulent flow and resolution transfer, but CNN baselines can be stronger, and Gray-Scott is less favorable under the tested noise level.
\end{enumerate}

\paragraph{Positioning.} Existing stochastic operator methods largely pay for distributional information in one of four ways. {athwise stochastic models learn a temporal evolution law \citep{Wilson_Borovitskiy_etal_2021, Salvi_Lemercier_etal_2022}, Bayesian or probabilistic neural operators quantify uncertainty over operator predictions \citep{Magnani_Kramer_etal_2022, Bulte_Scholl_etal_2025, Bayraktar_Feng_etal_2025}, chaos expansions represent stochastic inputs through basis coefficients \citep{Shi_Lin_etal_2026, Shi_Thompson_etal_2026}, and diffusion or flow models generate samples iteratively \citep{Kohl_Chen_etal_2024, Bortoli_Thornton_etal_2023, Yu_Li_etal_2026}. MNO targets the one-shot terminal conditional marginal
\[
    u_0 \mapsto \mathcal{L}(u_T \mid u_0),
\]
represented by a conditional mean field and a positive-semidefinite residual covariance factor, which is a much smaller object. The MNO target is weaker than pathwise stochastic-process modeling and less expressive than a general generative model, but this narrower target is also the source of the computational advantage. By separating a deterministic drift-like operator from a centered residual factor, MNO restores the aleatoric variance lost by mean-squared training while preserving the one-shot, grid-transferable inference that makes neural operators attractive.

\section{Related Work}
\label{sec:related}

\paragraph{UQ and neural operators.} Deterministic neural operators such as FNO and DeepONet learn resolution-transferable solution maps for PDEs \citep{Li_Kovachki_etal_2021,Lu_Jin_etal_2021,Kovachki_Li_etal_2024}. A growing literature augments neural operators with predictive uncertainty. Approximate Bayesian Neural Operators place Bayesian structure over operator surrogates, giving epistemic uncertainty for parametric PDE prediction \citep{Magnani_Kramer_etal_2022}. Probabilistic Neural Operators treat operator learning as a distributional prediction problem over functions, giving a broad functional-UQ framework \citep{Bulte_Scholl_etal_2025}. Active operator-learning methods use predictive uncertainty to guide data acquisition \citep{Winovich_Daneker_etal_2026}, and stochastic-maximum-principle approaches provide another probabilistic route to operator learning \citep{Bausback_Tang_etal_2026}.

\paragraph{Neural operators for stochastic dynamics.} Another line models stochastic dynamics directly. Neural SDEs \citep{Shen_Cheng_2025}, Neural CDEs \citep{Kidger_Morrill_etal_2020}, Neural rough differential equations \citep{Morrill_Salvi_etal_2021}, and Neural SPDEs \citep{Salvi_Lemercier_etal_2022} learn temporal stochastic evolution laws, making them natural when the object of interest is a full path distribution or a history-dependent forecast. Such models preserve a sequential view of the stochastic process, but inference and training generally involve temporal rollout, pathwise state, or repeated sampling. 

\paragraph{Neural extensions of classical SPDE methods.} Classical uncertainty quantification represents stochastic solutions through Monte Carlo, stochastic Galerkin methods, polynomial chaos, and Wiener-chaos expansions \citep{Higham_2001, Xiu_Karniadakis_2006, Nouy_2009, Doostan_Owhadi_2011}. These methods are principled and often highly accurate when the stochastic dimension, basis, and regularity are favorable, but they can become delicate for high-dimensional noise, rough drivers, or singular SPDEs. Recent neural Wiener-chaos methods combine this classical structure with neural operators, learning chaos coefficients or Wick-Hermite expansions for stochastic differential equations and SPDEs \citep{Shi_Lin_etal_2026, Neufeld_Schmocker_2026, Shi_Thompson_etal_2026}.

\paragraph{Generative samplers.} Generative and physics-guided samplers can represent richer conditional laws through iterative denoising \citep{Song_SohlDickstein_etal_2021}, transport \citep{Bortoli_Thornton_etal_2023}, or guided sampling \citep{Yu_Li_etal_2026}. These methods are better suited when high-fidelity samples, multimodality, hard physical constraints, or inverse-problem posterior structure are the central requirement. Inference for generative methods usually requires many network evaluations or task-specific guidance, which contrasts with the operator learning paradigm.

\section{Stochastic Operator Setup}
\label{sec:setup}

Let \(D\) be a bounded spatial domain and \(H=L^2(D)\). The stochastic operator-learning target is a map
\[
	u_0\mapsto\mathcal{L}(u_T\mid u_0)\in\mathcal{P}(H),
\]
beyond a single field-valued prediction. Standard neural operators trained with squared error recover only the conditional mean by the orthogonality of \(L^2\) projection,
\[
	f^\star(u_0)=\arg\min_f \mathbb{E}\|u_T-f(u_0)\|_H^2
	=\mathbb{E}[u_T\mid u_0].
\]
Any marginal learner must retain this mean term while representing the residual uncertainty left by squared-error training.

For the terminal marginal, the next object after the mean is covariance. Covariance operators are positive semidefinite by definition, and practical uncertainty surrogates should preserve that constraint under discretization. A finite factor \(B\) gives the minimal algebraic mechanism, since \(\Gamma=B^\top B\) is PSD automatically, has rank at most the number of factor rows, and yields diagonal variance by summing squared factor fields. The statement concerns moment parameterization, independent of any particular residual law.

Mercer theory motivates spectral covariance expansions for positive trace-class kernels, and pathwise Gaussian-process conditioning gives another modern route to scalable sample representations \citep{Wilson_Borovitskiy_etal_2021}. The finite-factor viewpoint is weaker than an explicit Mercer model, because it preserves positivity and rank control while leaving eigenfunctions and orthonormal factors unidentified.

Distributional accuracy is reported with empirical Wasserstein-2 distances for stochastic tasks and mean RMSE where the benchmark is deterministic or mean-field. For Gaussian laws, \(W_2\) jointly penalizes mean and covariance mismatch. For sampled ensembles we estimate it from the predicted and ground-truth terminal samples.

\section{Martingale Neural Operator (MNO)}
\label{sec:mno}

\paragraph{Setup.}
Given initial conditions \(\{u_0^{(i)}\}\) and ensemble realizations \(\{u_T^{(i,j)}\}_{j=1}^{N_{\mathrm{ens}}}\), we learn an operator
\[
	\mathcal{G}_{\theta,\phi} : L^2(D) \to \mathcal{P}(L^2(D)),
\]
mapping \(u_0 \mapsto \mathcal{L}(u_T \mid u_0)\) through a learned mean \(m_\theta\) and covariance operator \(\Gamma_\phi\).

Deterministic neural operators cheaply predict conditional means, while ensembles, diffusion models, and learned stochastic PDEs target fuller laws by drawing repeated samples. When training data already carry realized stochasticity along an ensemble axis, a single operator can emit both the mean and the residual covariance in one forward pass, at the same complexity class as a deterministic surrogate, and MNO learns that residual structure from data rather than fixing it through a hand-designed noise prior.

\paragraph{Classical motivation.}
The mean/residual split is classical. On a filtered probability space, a square-integrable special semimartingale admits a canonical decomposition into predictable finite variation and local-martingale parts \citep{Protter_2005}. For It\^o SPDEs, this is the familiar \(A_t=\int_0^t\mu(u_s,s)ds\) and \(M_t=\int_0^t\sigma(u_s,s)dW_s\), and MNO uses the terminal consequences of this split as an architectural prior. Fractional Brownian motion with \(H\neq1/2\) falls outside the semimartingale class \citep{Mandelbrot_VanNess_1968,Biagini_Hu_etal_2008}, so the rough-volatility experiments use the same form as a terminal mean/covariance factorization.

\paragraph{Doob-Meyer structural prior.}
The Doob-Meyer split thereby motivates a neural operator with three architectural commitments.
\begin{enumerate}
	\item \emph{Deterministic drift}. The drift \(\mathcal{A}_\theta(u_0,t)\) is parameterized as a deterministic function of \((u_0,t)\), which is stronger than the classical predictable condition but is the natural restriction for Markovian dynamics and for terminal marginals conditioned on \(u_0\).
	\item \emph{Zero-mean stochasticity}. The residual is represented through a parameterized covariance operator \(\Gamma_\phi(u_0,t)\) and a centered residual law, which yields \(\mathbb{E}[M_t\mid u_0]=0\) by construction, while the residual distribution remains a modeling choice, instantiated as Gaussian in our experiments.
	\item \emph{Consistent initialization}. Initial consistency is enforced by a temporal gate that drives both \(\mathcal{A}_\theta\) and \(\Gamma_\phi\) to zero at \(t=0\), so that \(u_t=u_0\) there.
\end{enumerate}

\paragraph{Realization.}
Given \(u_0 \in L^2(D)\) and query time \(t \in [0,T]\), MNO defines
\[
	m_\theta(u_0,t) = u_0 + \mathcal{A}_\theta(u_0,t),
	\qquad
	\Gamma_\phi(u_0,t) = B_\phi(u_0,t)^\top B_\phi(u_0,t),
\]
through a learned drift increment \(\mathcal{A}_\theta\) and a low-rank factor \(B_\phi\). For a \(C\)-channel field on a grid of \(N_x\) points, \(B_\phi(u_0,t)\) has shape \(r \times (C N_x)\). Equivalently, its entries are channel-aware basis fields \(B_k(c,x;u_0,t)\). The diagonal variance used for moment prediction is
\[
	\mathrm{Var}_\phi[u_t(c,x)\mid u_0]
	= \sum_{k=1}^r B_k(c,x;u_0,t)^2.
\]
For 2D domains, \(x\) is replaced by \((x,y)\) and the factor is flattened over \(C N_x N_y\). Samples take the form
\[
	u_t(x) = u_0(x) + \mathcal{A}_\theta(u_0,t)(x) + \sum_{k=1}^r B_k(x;u_0,t)\,\xi_k,
\]
for any centered coefficient vector \(\xi=(\xi_1,\dots,\xi_r)\) independent of \(u_0\) with \(\operatorname{Cov}(\xi)=I_r\). For any such \(\xi\),
\begin{align*}
	\mathbb{E}[u_t \mid u_0] &= m_\theta(u_0,t), \\
	\mathrm{Cov}(u_t \mid u_0) &= \Gamma_\phi(u_0,t),
\end{align*}
so the conditional mean and covariance are determined by the architecture alone.

The architecture has two FNO operator heads plus a temporal gate, with one head emitting \(\mathcal{A}_\theta\) and the other emitting \(rC\) residual-factor channels that reshape into \(B_\phi\).

\paragraph{Residual law.}
We instantiate \(\xi \sim \mathcal{N}(0, I_r)\) throughout this paper, giving
\[
	\mathcal{G}_{\theta,\phi}(u_0,t)
	= \mathcal{N}\big(m_\theta(u_0,t), \Gamma_\phi(u_0,t)\big).
\]
Gaussian \(\xi\) admits a closed-form negative log-likelihood and matches the moment-level metrics we report. Heavier-tailed or jump residuals would change the sampling law and training objective while preserving the positive semidefiniteness, rank control, and centered initialization supplied by the factorization.

\subsection{Covariance Operator}
We model the stochastic residual through a low-rank factor
\[
	M_\phi(u_0,t) = B_\phi(u_0,t)^\top \xi,\qquad \xi \sim \mathcal{N}(0,I_r).
\]

The head outputs \(rC\) channels in 1D and reshapes them as \(B \in \mathbb{R}^{r \times C \times N_x}\), while in 2D it outputs \(B \in \mathbb{R}^{r \times C \times N_x \times N_y}\). Flattening the channel and spatial dimensions gives \(\Gamma_\phi = B^\top B\), hence positive semidefiniteness is automatic and no eigenvalue parameterization or Gram-Schmidt step is needed. The variance field used by the loss and by reported moment metrics is the diagonal
\[
	\operatorname{diag}\Gamma_\phi(c,x) = \sum_{k=1}^r B_k(c,x)^2.
\]
The channel-aware shape matters in multi-field systems such as Gray-Scott, where the \(u\) and \(v\) channels can have distinct variance fields.

\subsection{Drift Operator}

We parameterize the drift operator \(\mathcal{A}_\theta(u_0,t)\) with an FNO backbone applied to the initial condition concatenated with a constant time channel. All reported experiments use this implementation. It keeps the usual resolution-transfer behavior of Fourier neural operators while separating the deterministic mean head from the stochastic factor head. A temporal gate \(g(t)=1-\exp(-|\alpha|t)\) multiplies both heads, enforcing \(A(u_0,0)=0\) and \(B_\phi(u_0,0)=0\).

\subsection{Theoretical Guarantees}

We summarize the architectural guarantees and approximation scope of the implemented MNO. Supporting details are deferred to the appendices.

\begin{theorem}[Drift Approximation Scope, informal]\label{thm:drift-approx-informal}
	Under the standard approximation assumptions for FNO-type architectures on compact families of sufficiently regular functions, the drift head can approximate continuous solution operators on bounded spatial domains. The present guarantee is the usual neural-operator approximation scope of the implemented FNO drift head. Details and limitations are summarized in Appendix~\ref{app:separable-kernel}, with a separable/Mercer formulation presented as an auxiliary prior in Appendix~\ref{app:auxiliary-priors}.
\end{theorem}

\begin{theorem}[Covariance Operator Approximation, informal]\label{thm:cov-approx-informal}
	For any positive semidefinite covariance operator whose effective rank is at most \(r\), a factor \(B\) exists such that \(\Gamma=B^\top B\). Truncating a trace-class covariance to rank \(r\) gives the usual spectral tail error, and MNO learns such a factor directly, bypassing explicit Mercer eigenpair recovery. The architectural guarantee is positivity and rank control, while exact eigenfunction recovery lies outside the claim. Full discussion appears in Appendix~\ref{app:covariance}.
\end{theorem}

\begin{theorem}[Resolution Invariance, informal]\label{thm:resolution-informal}
	Let \(\mathcal{G}_{\theta,\phi}^{(n)}\) denote the MNO operator discretized on a grid of size \(n\). When the FNO heads are evaluated consistently across grids and the learned factor fields have stable continuous limits, the predicted mean and variance fields can be evaluated zero-shot on new resolutions. This architectural design goal is tested empirically in \S\ref{sec:experiments}, while implementation-independent \(O(n^{-2})\) rates for arbitrary learned weights are outside our scope.
\end{theorem}

The implemented model guarantees zero initial residuals, zero-mean sampling, positive semidefinite low-rank covariance, and channel-aware variance fields, while the empirical question is whether these guarantees translate into accurate moments, calibrated covariance, and resolution-stable inference.

\section{Experiments}
\label{sec:experiments}


The experiments ask, in order, whether the residual head improves terminal distributional accuracy, whether the same factorization remains useful outside the semimartingale regime, and whether the moment representation preserves the computational advantage of neural operators. Table~\ref{tab:application-results} gives the consolidated answer, with the main cases discussed in the following subsections.

\begin{table}[!t]
	\caption{Consolidated benchmark results. MNO is evaluated against the most relevant comparator for each task family, using stochastic-surrogate baselines (Neural SPDE, Neural SDE) for 1D SPDE tasks, FNO for 2D operator tasks, and a conditional diffusion model with an FNO score backbone for generative efficiency. Lower is better for all metrics. Gray-Scott is included as a less favorable 2D case, with full 2D results in Appendix~\ref{app:additional-results}. The \(\phi^4\) row also serves as the SPDEBench \citep{Zhu_Li_etal_2026} evaluation, using the same data and an independently verified wrapper as described in Appendix~\ref{app:additional-results}.}
	\label{tab:application-results}
	\begin{center}
		\begin{small}
			\begin{sc}
				\resizebox{\linewidth}{!}{%
					\begin{tabular}{llcccc}
						\toprule
						Benchmark & Metric & Baseline method & Baseline value & MNO & \(\Delta\) \\
						\midrule
						Stochastic Burgers & \(W_2\) & Neural SPDE & 0.6483 & \textbf{0.0095} & \(68\times\) \\
						Rough Vol (\(H=0.1\)) & \(W_2\) & Neural SDE & 0.0668 & \textbf{0.0257} & \(2.6\times\) \\
						\(\phi^4\) / SPDEBench & \(W_2\) & Neural SPDE & 0.6572 & \textbf{0.0055} & \(120\times\) \\
						1D superresolution (\(n=128\)) & mean RMSE & Neural SPDE & 0.1957 & \textbf{0.0407} & \(4.8\times\) \\
						Generative efficiency & inference (s) & Diffusion (NFE\(=\)25) & 0.0213 & \textbf{0.0070} & \(3\times\) \\
						2D turbulent flow & mean RMSE & FNO & \textbf{0.6583} & 0.6687 & parity \\
						2D resolution transfer (\(n=128\)) & mean RMSE & FNO & 8.7421 & \textbf{8.4394} & parity \\
						2D Gray-Scott & mean RMSE & FNO & \textbf{0.0054} & 0.0345 & \(-6.3\times\) \\
						\bottomrule
					\end{tabular}
				}
			\end{sc}
		\end{small}
	\end{center}
	\vskip -0.1in
\end{table}

\subsection{Experimental Setup}

\subsubsection*{Baselines.}
We select baselines according to the scientific claim each experiment targets. For 1D stochastic-surrogate tasks, including Burgers, \(\phi^4\), and 1D superresolution, the primary comparator is Neural SPDE \citep{Salvi_Lemercier_etal_2022}, which directly models SPDE solutions, while Wiener-chaos-style and SDENO-style surrogates are retained in the appendix for completeness. For rough volatility, the primary comparators are Neural SDE \citep{Shen_Cheng_2025} and Neural CDE \citep{Kidger_Morrill_etal_2020}, which are the most principled sequential baselines for non-Markovian stochastic dynamics. For 2D operator tasks such as turbulent flow and resolution transfer, FNO \citep{Li_Kovachki_etal_2021} serves as the main baseline, with CNN-based baselines such as ResNet and U-Net appearing in the appendix. For the generative-efficiency comparison, the baseline is a conditional diffusion model with an FNO score backbone, trained under the same wall-clock budget as MNO and evaluated across an NFE sweep of \(\{25, 50, 100\}\). The appendix structural diagnostics use Noise-FNO as the baseline for resolution invariance.

\subsubsection*{Metrics.}
For stochastic tasks where distributional quality is the primary claim, the headline metric is the empirically estimated spatially-averaged Wasserstein-2 distance (\(W_2\)) between the predicted terminal ensemble and the ground-truth ensemble. For operator-learning tasks where mean-field accuracy is the relevant comparison (2D turbulence, resolution transfer), we use mean RMSE over test trajectories. The generative-efficiency row reports inference wall-clock time (seconds per evaluation on matched hardware), with \(W_2\) reported as a quality control. Variance RMSE, used for the appendix resolution-invariance and uncertainty diagnostics, measures the Frobenius distance between predicted and ground-truth variance fields in \emph{normalized} coordinates (matching the MNO training objective). Physical scaling follows by multiplying by the squared target standardization.

\subsubsection*{Data.}
Training budgets are heterogeneous across task families. For the standard 1D stochastic surrogate tasks (Burgers, \(\phi^4\), superresolution), we use 1000 training trajectories, 200 test trajectories, 192 ensemble members for ground-truth moment estimation, and 120 training epochs. For the generative-efficiency experiment, both MNO and the diffusion baseline are trained on the stochastic Burgers terminal ensemble under a matched 60-second wall-clock budget. With this budget MNO reaches its 500-epoch cap in 23.1 seconds and the diffusion baseline completes 21 epochs in 61.6 seconds, reflecting the per-epoch cost difference of the flattened-ensemble diffusion target. For the 2D runs reported here we use 1000 training trajectories, 100 test trajectories, 96 ensembles, and 120 epochs. Rough-volatility trajectories are generated via Davies-Harte fractional Brownian motion simulation at each \(H\) value in the grid, with the same 1000/200 train/test split, 192 ensembles, 120 epochs, and five evaluation seeds. Full hyperparameter details (architecture widths, learning rates, modes) are in Appendix~\ref{app:training}.

\subsection{Stochastic Burgers' Equation}

\begin{figure}[t]
	\centering
	\includegraphics[width=\columnwidth]{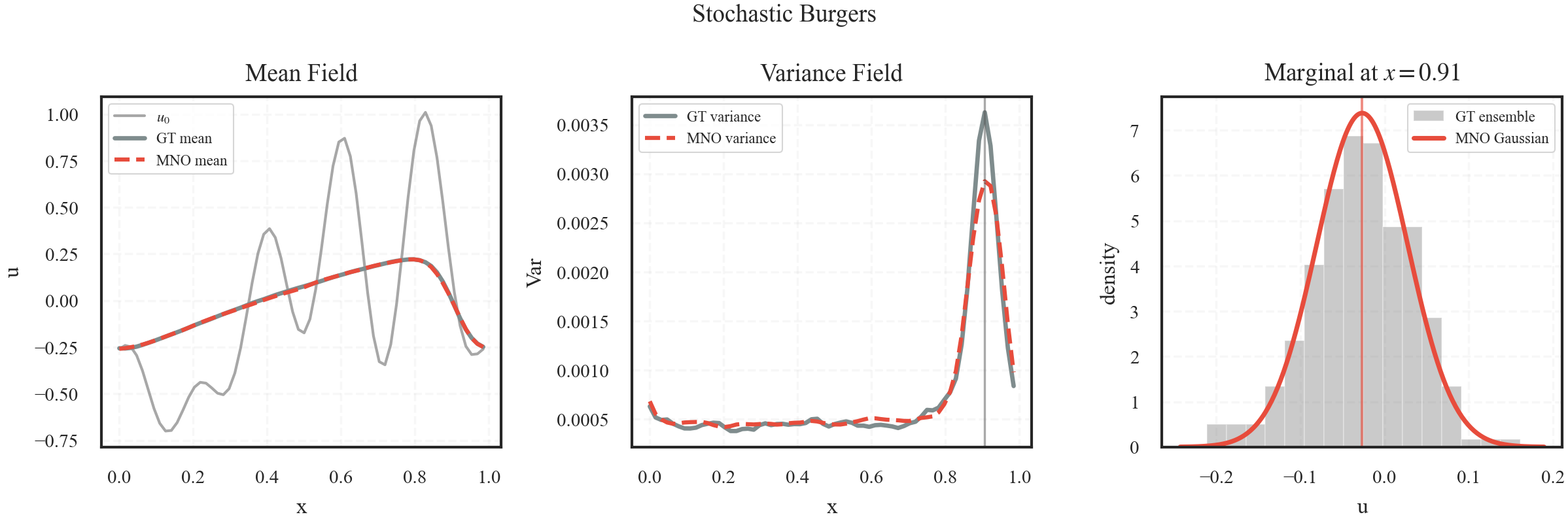}
	\caption{Stochastic Burgers' Equation results. For a representative test initial condition, the figure shows the terminal mean field, variance field, and pointwise marginal density. MNO recovers a nontrivial residual variance profile while remaining competitive with deterministic baselines on mean prediction.}
	\label{fig:case-burgers}
\end{figure}

MNO recovers a nontrivial variance profile from a single forward pass while matching the mean accuracy of a deterministic baseline. On the distributional metric, MNO achieves \(W_2 = 0.0095\), a \(68\times\) reduction over Neural SPDE (\(W_2 = 0.6483\)), which struggles to capture the stochastic structure despite end-to-end SPDE training. On the mean-prediction metric, MNO achieves mean RMSE \(= 0.0230\), close to FNO (\(0.0217\)). FNO lacks a variance head, and although its \(W_2 = 0.0179\) is strong because the mean is accurate, it provides no stochastic residual or variance field. The mean RMSE parity shows that the martingale head leaves the drift head essentially intact on this benchmark.

\subsection{Rough Volatility Dynamics}
Rough volatility is a deliberate stress test of the motivation. Fractional drivers with \(H<1/2\) fall outside the semimartingale setting, so the Doob-Meyer interpretation becomes a structural analogy, and the question becomes whether the same disciplined split can capture the terminal uncertainty profile induced by roughness. In this regime, MNO is a one-shot Gaussian approximation to the terminal marginal, with the mean trained by squared prediction error and the residual factor trained through Gaussian negative log-likelihood, variance consistency, residual centering, and factor regularization.

\begin{figure}[t]
	\centering
	\includegraphics[width=\columnwidth]{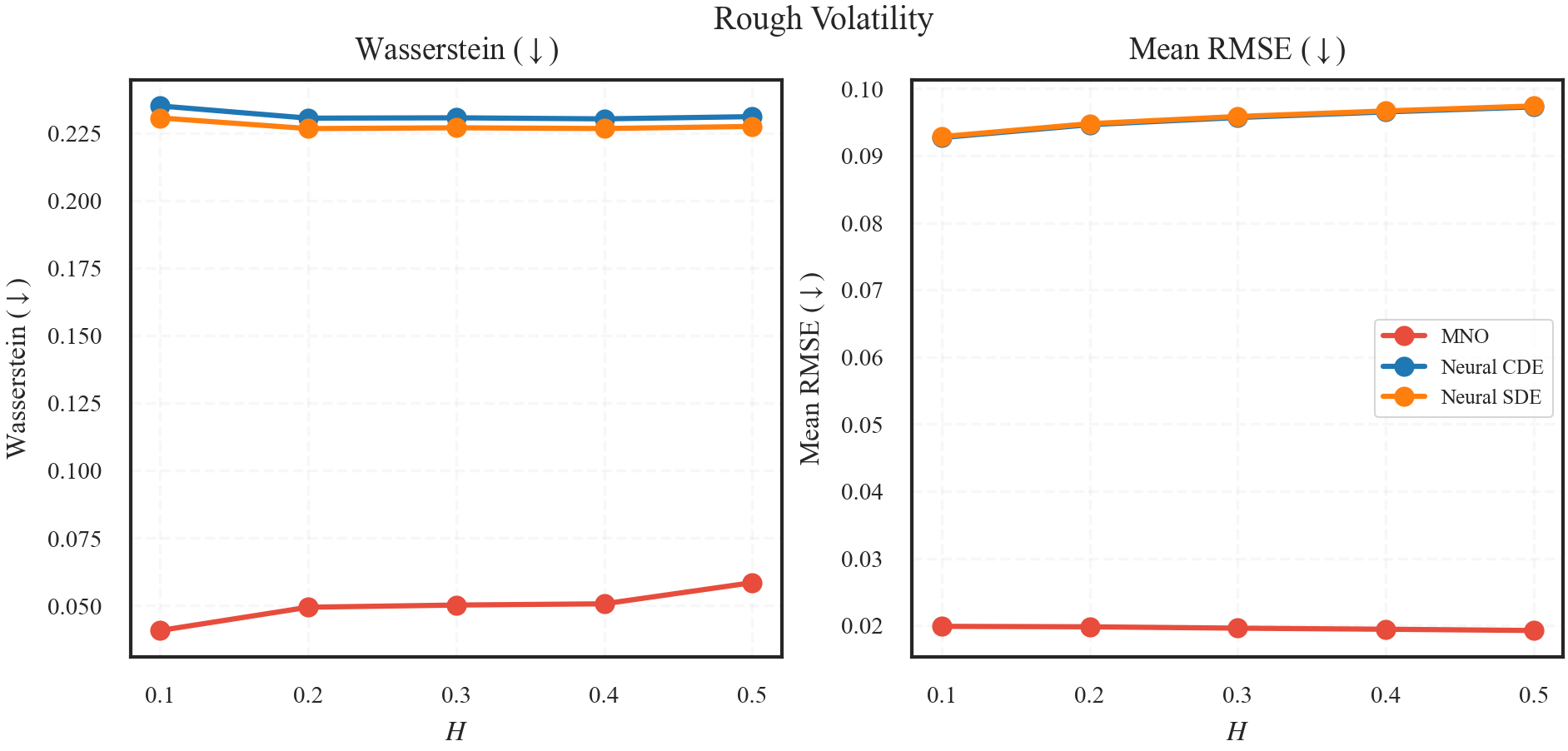}
	\caption{Rough volatility results: comparison of MNO against sequential stochastic baselines across varying Hurst parameters \(H\). MNO remains accurate in the deeply rough regime while targeting terminal marginals over pathwise roughness recovery.}
	\label{fig:case-rough-vol}
\end{figure}

We model the log-volatility process of rough Heston/Bergomi style \citep{Heston_1993, Bergomi_2004, Bayer_Friz_etal_2015} as \(d\log\sigma_t = \eta dW^H_t\), where \(W^H\) is a fractional Brownian motion with Hurst parameter \(H\).
We evaluate across \(H \in \{0.1, 0.2, 0.3, 0.4, 0.5\}\), spanning from the deeply rough regime through standard Brownian motion. The headline comparison uses \(H = 0.1\), where non-Markovian effects are strongest and sequential baselines are most challenged.

At \(H = 0.1\), MNO achieves mean \(W_2 = 0.0257\) on the terminal marginal across five seeds, compared to Neural SDE (\(W_2 = 0.0668\), a \(2.6\times\) gap) and Neural CDE (\(W_2 = 0.0675\)). The sequential baselines are optimized for pathwise fidelity and require the full trajectory history to estimate roughness, but in this fixed benchmark they produce worse terminal marginals than MNO's one-shot prediction. The measured claim is correspondingly narrow, since when terminal marginal distribution matters more than pathwise roughness recovery, MNO's structured Gaussian approximation can be the more accurate object. Across the tested \(H\)-sweep MNO remains ahead of the sequential baselines.

\subsection{Additional Benchmarks}

Beyond the two case studies above, Table~\ref{tab:application-results} summarizes the remaining headline benchmarks, and Figures~\ref{fig:phi4-results}-\ref{fig:generative-results} provide visualizations for three representative results.

\textbf{\(\phi^4\) field theory.} MNO achieves \(W_2 = 0.0055\) vs.\ Neural SPDE \(0.6572\) (Figure~\ref{fig:phi4-results}). The double-well \(\phi^4\) potential produces a broad terminal law, and MNO's low-rank residual factor captures its variance profile substantially better than the stochastic-surrogate baselines in the reported run.

\begin{figure}[t]
	\centering
	\includegraphics[width=0.8\columnwidth]{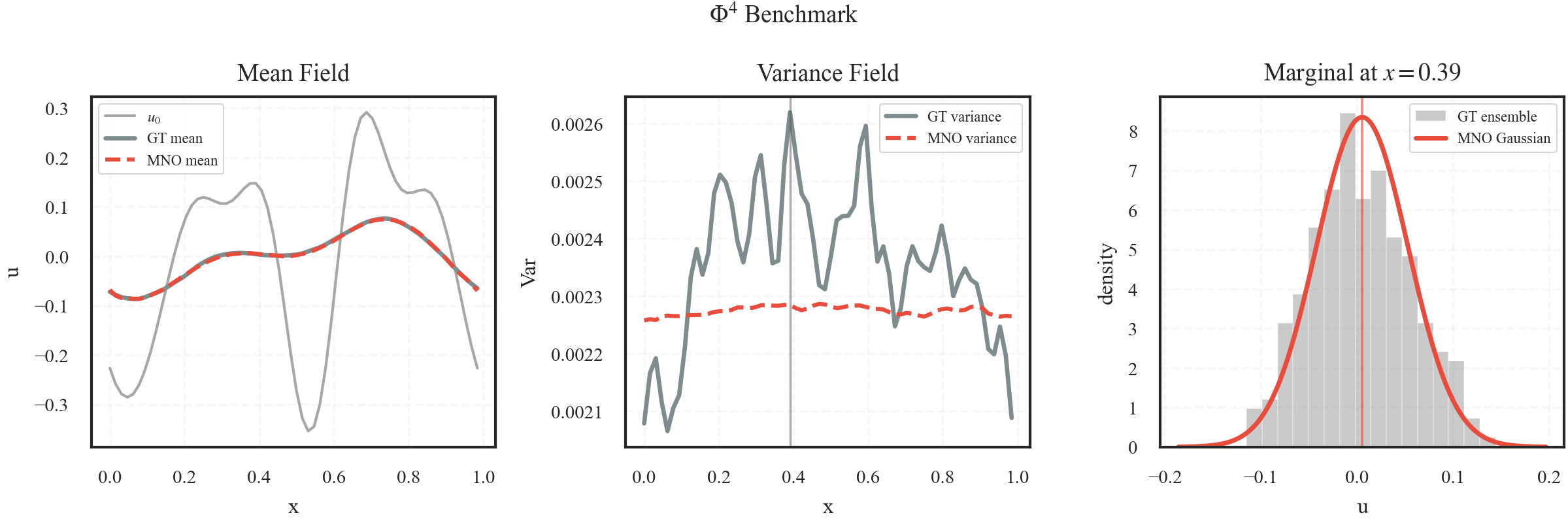}
	\caption{\(\phi^4\) field theory benchmark (SPDEBench). MNO achieves \(W_2 = 0.0055\) while Neural SPDE obtains \(W_2 = 0.6572\). The figure shows the learned terminal mean, variance field, and a representative marginal density induced by MNO's low-rank residual factor.}
	\label{fig:phi4-results}
\end{figure}

\textbf{1D zero-shot super-resolution.} MNO trained at resolution 32 generalizes zero-shot to resolution 128, achieving mean RMSE \(= 0.0407\) vs.\ Neural SPDE \(0.1957\) (Figure~\ref{fig:superres-results}). Neural SPDE degrades monotonically (\(0.102 \to 0.196\)) across the resolution sweep because its pathwise discretization is tied to the training grid. MNO's mean field transfers cleanly, and its variance field remains in the same error scale. At the finest grid, other baselines still lead on either mean or variance.

\begin{figure}[t]
	\centering
	\includegraphics[width=0.8\columnwidth]{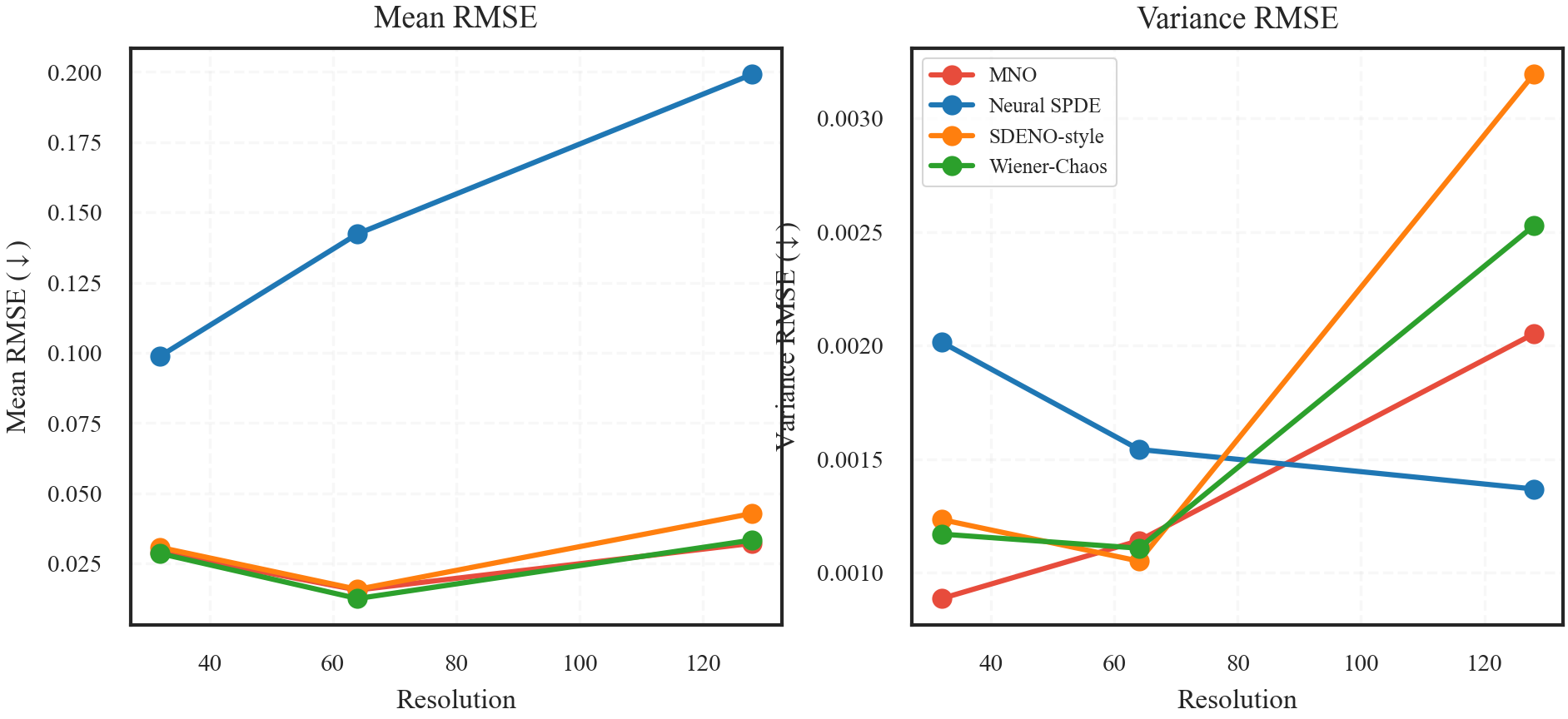}
	\caption{1D zero-shot super-resolution. MNO (trained at resolution 32) achieves stable or improving mean RMSE at resolutions 32/64/128, while Neural SPDE degrades monotonically. The variance row is reported as an uncertainty diagnostic, with no headline superiority claim at every resolution.}
	\label{fig:superres-results}
\end{figure}

\textbf{Generative efficiency.} MNO produces moments in \(7.0\times 10^{-3}\) s vs.\ \(2.1\times 10^{-2}\) s for a conditional diffusion baseline with an FNO score backbone at NFE\(=25\), where both models are trained under a matched 60-second wall-clock budget, reaching 500 MNO epochs in 23.1 seconds and 21 diffusion epochs in 61.6 seconds because of the per-epoch cost difference of the flattened-ensemble diffusion target. This gives a \({\sim}3\times\) inference speedup while achieving lower \(W_2\) (\(0.0081\) vs.\ \(0.1589\), Figure~\ref{fig:generative-results}). The comparison measures moment prediction time for MNO against sample generation time for diffusion, which is the operationally relevant comparison when applications require mean and covariance estimates such as risk envelopes or UQ bounds.

\begin{figure}[t]
	\centering
	\includegraphics[width=\columnwidth]{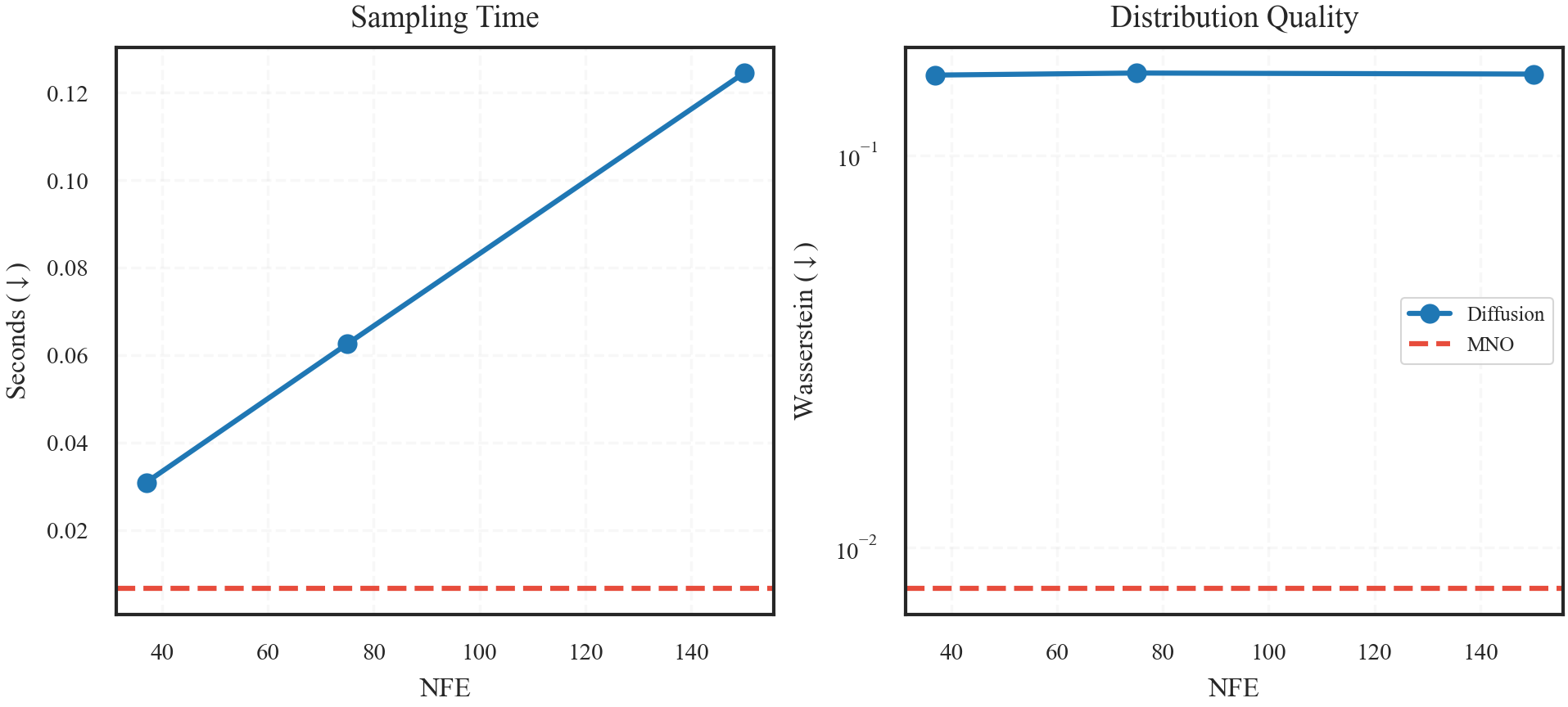}
	\caption{Generative efficiency comparison on stochastic Burgers'. MNO predicts terminal moments in a single forward pass (\(7.0\times10^{-3}\) s). The diffusion baseline requires \(2.1\times10^{-2}\) s at NFE\(=25\). MNO also achieves lower \(W_2\) (\(0.0081\) vs.\ \(0.1589\)) under a matched 60-second wall-clock training budget.}
	\label{fig:generative-results}
\end{figure}

\textbf{2D turbulent flow.} MNO is close to FNO on mean RMSE (\(0.6687\) vs.\ \(0.6583\)), while U-Net is substantially better on this particular task (\(0.0706\)). We treat this result as a dimensionality proof-of-concept for the channel/spatial low-rank factor, with no state-of-the-art mean-prediction claim. Full 2D results, including Gray-Scott, appear in Appendix~\ref{app:additional-results}.

\section{Discussion and Limitations}
\label{sec:discussion}

MNO learns terminal marginals, and pathwise stochastic-process modeling is outside its scope. The Martingale Mirror diagnostic makes this boundary visible because one-shot terminal marginals can be accurate while autoregressive reuse whitens rough temporal correlations toward Brownian scaling. The Doob-Meyer interpretation is valid as structural motivation in semimartingale regimes, whereas on rough volatility the correct reading is a useful terminal mean/covariance factorization.

The second limitation is expressivity of the experimental residual law. The architecture only requires a centered coefficient vector with known covariance, but our runs use a Gaussian, low-rank residual, which limits heavy tails, jumps, and multimodal laws. Other residual laws can be paired with the same \(B_\phi\), though they require different likelihoods, sampling rules, and validation targets. Because both heads are FNO-based, high-frequency deterministic structure can also be misplaced into variance, which is one plausible explanation for the weaker Gray-Scott result. MNO enforces terminal centering and PSD covariance while leaving conservation laws and the full filtration-level martingale tower property unenforced.




\bibliographystyle{plainnat}
\bibliography{references}

\appendix

\section{Structural Diagnostics}
\label{app:verification-suite}

This section reports implementation diagnostics for residual centering, autoregressive reuse, resolution transfer, head separation, and uncertainty decomposition. These checks record measured behavior in the reported runs, while pathwise martingality and global identifiability require stronger evidence.

\subsection{Diagnostic I: Residual Centering and Scaling}

\begin{figure}[t]
	\centering
	\includegraphics[width=0.8\textwidth]{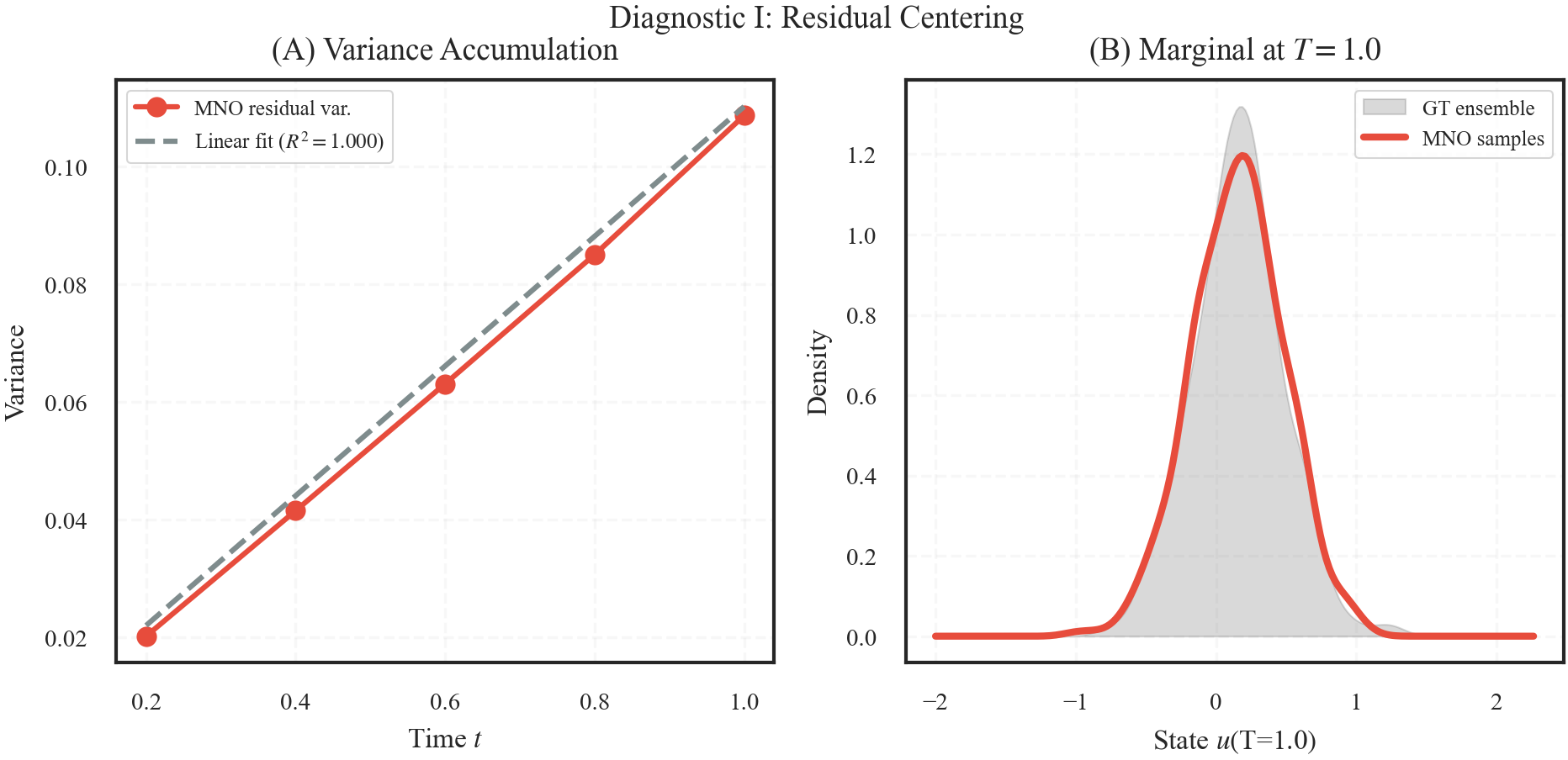}
	\caption{Diagnostic I, residual centering and scaling. The reported diagnostic records (i) near-zero temporal mean, (ii) approximately linear variance growth in the Brownian diagnostic, and (iii) weak correlation between selected residual increments and past-state features. These measurements check terminal residual behavior, short of full filtration-level martingality.}
	\label{fig:prop0}
\end{figure}

We evaluate whether the learned residual \(M_t = u_t - \hat{A}_t\) exhibits the terminal consequences expected of a centered residual by checking near-zero residual mean at all evaluation times, approximately linear variance growth in the Brownian diagnostic, and low empirical correlation between selected increments \(M_{t+\Delta t} - M_t\) and measured past-state features. All three checks pass in the reported run.

\subsection{Diagnostic II: Martingale Mirror and Non-Markovianity}
\label{subsec:mirror}
This subsection studies an \emph{autoregressive evaluation mode} that is distinct from the main one-shot definition of MNO. After training the one-shot operator \((u_0,t) \mapsto p(u_t \mid u_0)\), we recursively reuse the same model with the current predicted state as input, \(\hat{u}_{t_{n+1}} \sim \mathcal{G}_{\theta,\phi}(\hat{u}_{t_n}, \Delta t)\). Conditioning on the current state appears only in this rollout analysis, whereas the core model in the main text conditions on \(u_0\).

\begin{figure}[t]
	\centering
	\includegraphics[width=1\textwidth]{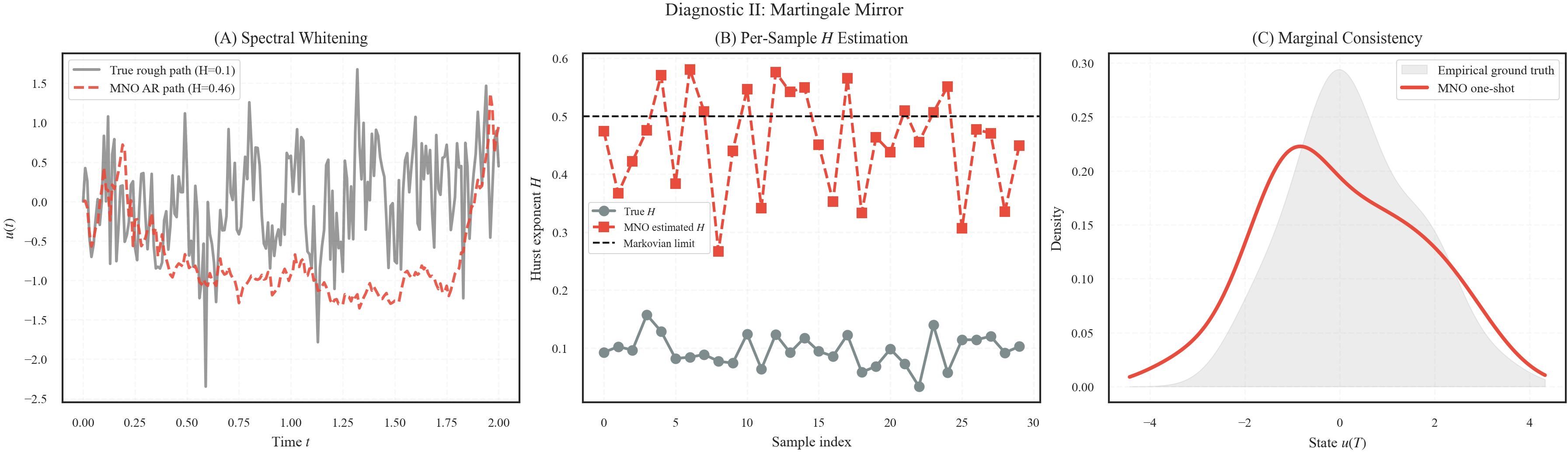}
	\caption{Diagnostic II, Martingale Mirror. The left panel shows one-shot MNO conditioned on \(u_0\) matching the plotted terminal marginal density for fBm with \(H=0.1\). The center panel shows autoregressive reuse producing a wider density, with variance growing closer to linearly than to \(t^{2H}=t^{0.2}\). The right panel shows AR paths with \(\hat{H}\approx0.486\), while ground-truth paths have \(\hat{H}\approx0.101\), indicating that repeated Markovian stepping changes the temporal scaling.}
	\label{fig:martingale-mirror}
\end{figure}

This diagnostic separates one-shot marginal accuracy from autoregressive path behavior for non-Markovian dynamics such as rough fractional Brownian motion (\(H<0.5\)).

As shown in Figure~\ref{fig:martingale-mirror}, we contrast one-shot evaluation, where MNO maps \(0 \to T\) in a single forward pass, with autoregressive stepping, where MNO iterates \(u_{t+\Delta t} = \mathcal{G}(u_t, \Delta t)\).

\textbf{Autoregressive scaling.} In the center panel, the one-shot model closely matches the plotted terminal marginal density at \(T\). The autoregressive rollout is wider because repeated Markovian stepping adds residual variance in a way closer to linear growth, whereas the target rough process has variance scaling \(t^{2H}=t^{0.2}\).

\textbf{Path scaling.} The right panel shows \(\hat{H}\approx0.486\) for AR paths, while the ground-truth paths have \(\hat{H}\approx0.101\). The measurement matches the training target, since MNO learns one-shot terminal marginals and leaves rough path-law preservation to a different objective.

\subsection{Additional Structural Diagnostics}

\begin{figure}[t]
	\centering
	\includegraphics[width=0.8\textwidth]{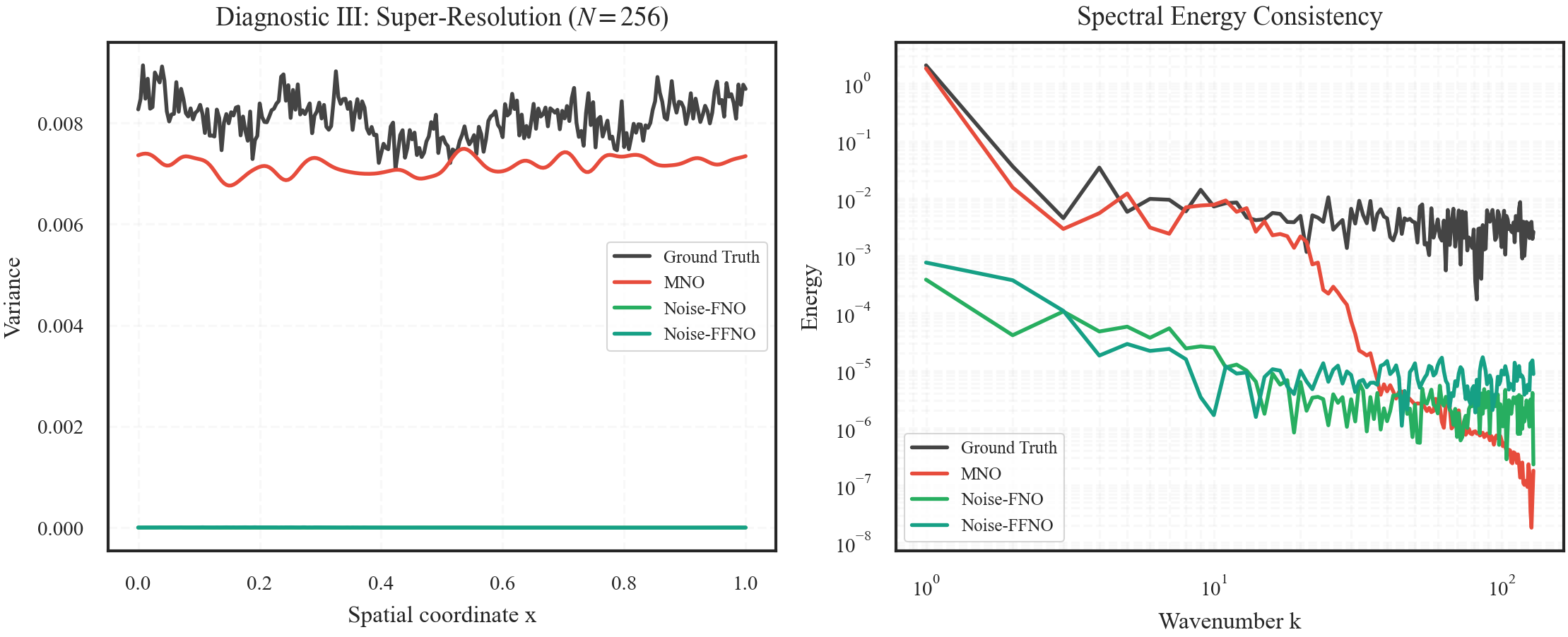}
	\caption{Diagnostic III, resolution-transfer variance. MNO is trained at resolution 64 and evaluated zero-shot through resolution 256. The left panel compares the highest-resolution variance fields, and the right panel compares Fourier energy spectra. At resolution 256, variance-field RMSE is \(0.00167\) for MNO vs.\ \(0.00853\) for Noise-FNO. The spectrum is reported qualitatively, with MNO tracking the low-frequency variance energy more closely while fine-scale discrepancies remain.}
	\label{fig:res-invariance}
\end{figure}

\textbf{Resolution-transfer variance.} We train MNO at spatial resolution 64 and evaluate zero-shot at resolutions 64, 128, and 256 on a 1D stochastic superresolution task. The comparison baselines are Noise-FNO and Noise-FFNO, which inject noise channels to produce stochastic outputs. Table~\ref{tab:prop2-full} reports the variance-field errors across resolutions.

\begin{table}[H]
	\centering
	\caption{Resolution-transfer variance diagnostic. Variance-field RMSE is computed against the empirical variance field at each evaluation resolution.}
	\label{tab:prop2-full}
	\begin{tabular}{lcccc}
		\toprule
		Resolution & True var.\ mean & MNO RMSE & Noise-FNO RMSE & Noise-FFNO RMSE \\
		\midrule
		64 & 0.00686 & 0.000553 & 0.006874 & 0.006877 \\
		128 & 0.00728 & 0.000449 & 0.007289 & 0.007290 \\
		256 & 0.00850 & 0.001666 & 0.008534 & 0.008534 \\
		\bottomrule
	\end{tabular}
\end{table}

\begin{figure}[t]
	\centering
	\includegraphics[width=0.6\textwidth]{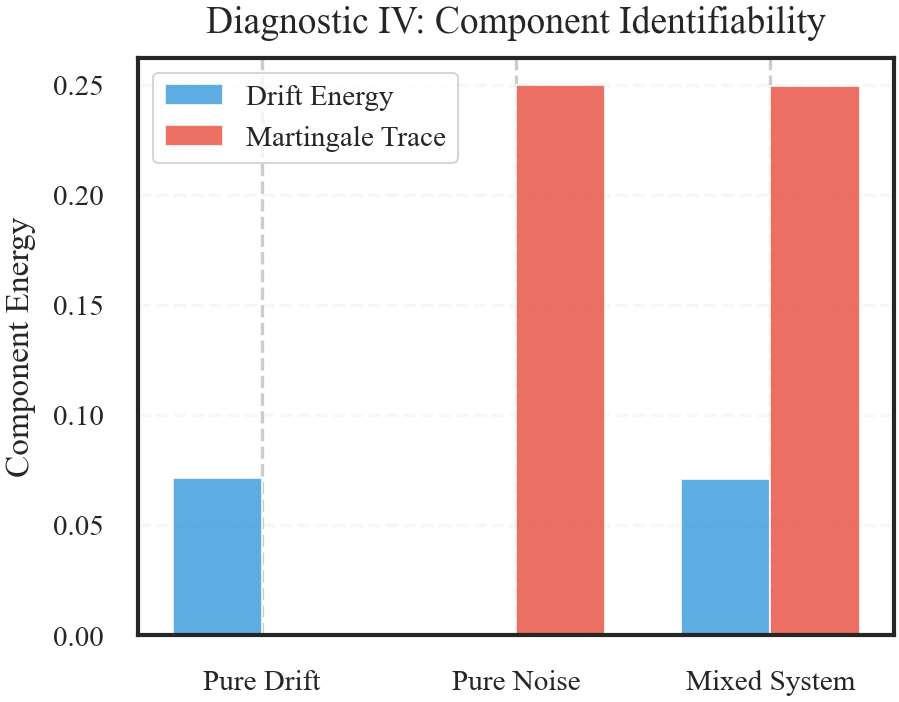}
	\caption{Diagnostic IV, head-separation audit. Across the three synthetic cases the trained MNO recovers the correct head energies, with drift-energy large only when drift is present (cases A and C), and covariance-energy large only when noise is present (cases B and C). The two heads are separable in this controlled setting.}
	\label{fig:identifiability}
\end{figure}
\nopagebreak

\begin{table}[H]
	\centering
	\caption{Head-separation diagnostic values. Cases A, B, and C denote pure drift, pure noise, and drift-plus-noise synthetic data, respectively.}
	\label{tab:head-separation}
	\begin{tabular}{lcc}
		\toprule
		Case & Drift-head energy & Covariance-head energy \\
		\midrule
		A & 0.07143 & \(7.22\times10^{-6}\) \\
		B & \(4.20\times10^{-6}\) & 0.24993 \\
		C & 0.07119 & 0.24986 \\
		\bottomrule
	\end{tabular}
\end{table}

\textbf{Head separation.} To test whether the two heads can be separated in simple regimes, we trained MNO on synthetic datasets generated by controlled drift/noise cases. The diagnostic shows the expected pattern: the drift head accumulates energy only when the data has a drift component (cases A and C), and the covariance head accumulates energy only when the data has a noise component (cases B and C). Head energies for the off-target case in each pair sit four orders of magnitude lower than the on-target value, so the two channels are empirically separable in this controlled setting, consistent with the identifiability statement up to the rotational ambiguity discussed in Appendix~\ref{app:covariance}.

\begin{figure}[t]
	\centering
	\includegraphics[width=0.6\textwidth]{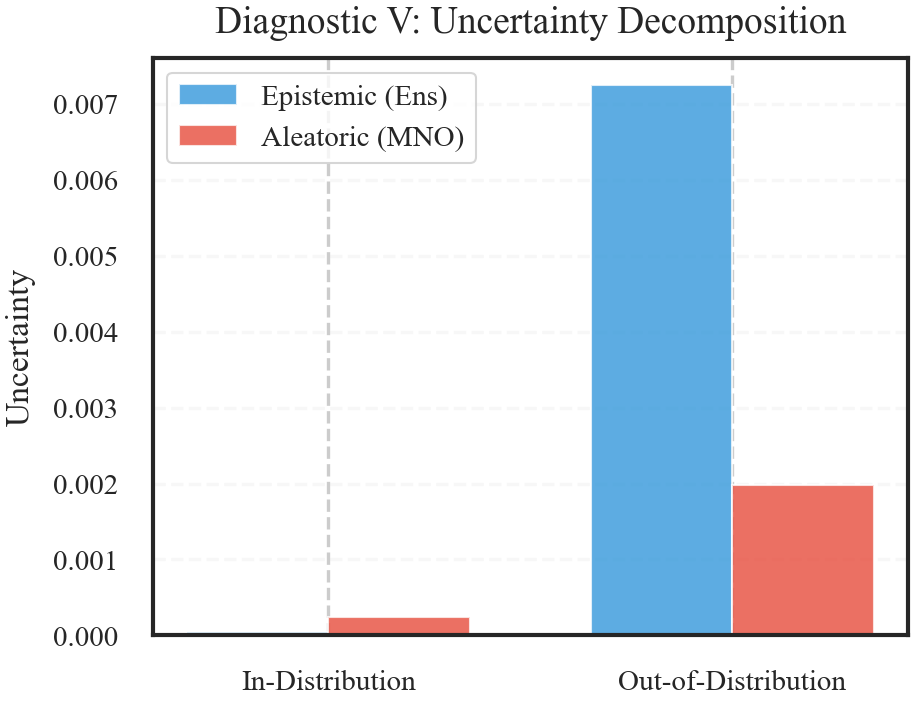}
	\caption{Diagnostic V, uncertainty decomposition. The plotted bars compare an under-trained ensemble proxy for in-distribution/out-of-distribution epistemic variance with MNO's predicted in-distribution aleatoric variance. The true aleatoric scalar is \(2.36\times10^{-4}\), while MNO predicts \(2.47\times10^{-4}\). The ensemble proxy is qualitative.}
	\label{fig:uncertainty}
\end{figure}
\nopagebreak

\textbf{Uncertainty decomposition.} MNO predicts an in-distribution aleatoric variance of \(2.47 \times 10^{-4}\), compared to the ground-truth aleatoric variance of \(2.36 \times 10^{-4}\). The epistemic proxy comes from a deliberately under-trained ensemble and should be read qualitatively.

\section{Doob-Meyer Decomposition}
\label{app:doob-meyer}

We invoke the classical theorem. The precise statement used by the paper is the canonical decomposition for square-integrable special semimartingales under the usual filtration hypotheses \citep{Doob_1937,Meyer_1962,Meyer_1963,Protter_2005,Beiglbock_Schachermayer_etal_2012}. In that regime,
\[
	u_t=u_0+A_t+M_t,
\]
where \(A\) is predictable finite variation and \(M\) is a local martingale. For an It\^o SPDE,
\[
	du_t=\mu(u_t,t)\,dt+\sigma(u_t,t)\,dW_t,
\]
the split is explicit as
\[
	A_t=\int_0^t\mu(u_s,s)\,ds,\qquad
	M_t=\int_0^t\sigma(u_s,s)\,dW_s.
\]
The stochastic integral is a martingale under the standard square-integrability conditions, and the drift integral is finite variation. Hilbert-space SPDE versions follow the Da Prato-Zabczyk framework when the noise and coefficients satisfy the corresponding well-posedness and Hilbert-Schmidt assumptions \citep{DaPrato_Zabczyk_2014}.

MNO learns the terminal map \(u_0\mapsto(m_\theta(u_0,t),B_\phi(u_0,t))\), which has the same mean/residual form in semimartingale cases. Infinitesimal coefficient recovery and data-driven pathwise certification lie outside the model's scope. With the Gaussian residual used in this paper, the same map becomes a Gaussian terminal-marginal approximation beyond semimartingales.

\section{FNO Drift Head and Approximation Scope}
\label{app:separable-kernel}

\subsection{Implemented Drift Head}
The MNO drift head is an FNO map
\[
	\mathcal{A}_\theta(u_0,t) = \operatorname{FNO}_\theta([u_0,t]),
\]
where \(t\) is concatenated as a constant channel over the spatial grid. The output is multiplied by the temporal gate
\[
	g(t)=1-\exp(-|\alpha|t),
\]
so \(A_\theta(u_0,0)=0\). In configurations with \texttt{shared\_backbone=True}, drift and residual heads share the FNO feature extractor and use separate projections. In configurations with \texttt{shared\_backbone=False}, they use independent FNO backbones.

\subsection{Approximation Scope}
The paper relies on the standard approximation and discretization behavior of Fourier neural operators for the drift component. The resulting claim is narrower and cleaner.
\begin{itemize}
	\item the drift head is a deterministic neural operator conditioned on \((u_0,t)\).
	\item the gate enforces exact initialization at \(t=0\).
	\item the architecture can be evaluated on different grids in the same manner as an FNO.
	\item approximation quality is empirical in this paper, with no explicit Barron-rate proof for the implemented weights.
\end{itemize}

This scope is sufficient for the experimental claims made here, because the mean-field results measure whether the learned drift is accurate, while the stochastic metrics measure whether the residual factor improves terminal marginal predictions.

\section{Covariance Operator Properties}
\label{app:covariance}

We model the stochastic residual covariance through a learned low-rank factor, with explicit Mercer eigenpairs appearing only in the auxiliary prior of Appendix~\ref{app:auxiliary-priors}.

\subsection{Low-Rank Factor Parameterization}
For a \(C\)-channel field discretized on \(N\) spatial points, the martingale head outputs
\[
	B_\phi(u_0,t) \in \mathbb{R}^{r \times C \times N}.
\]
Flattening the last two axes gives \(B_\phi \in \mathbb{R}^{r \times CN}\), and samples are drawn by
\[
	M_\phi(u_0,t)=B_\phi(u_0,t)^\top \xi,\qquad \xi\sim\mathcal{N}(0,I_r).
\]
Thus
\[
	\Gamma_\phi(u_0,t)=\operatorname{Cov}[M_\phi\mid u_0]=B_\phi(u_0,t)^\top B_\phi(u_0,t).
\]
For 2D fields the same construction uses \(B_\phi \in \mathbb{R}^{r \times C \times N_x \times N_y}\) and flattens over \(C N_x N_y\).

\subsection{Structural Guarantees}
The factor parameterization gives three guarantees used by the experiments.
\begin{enumerate}
	\item \textbf{Positive semidefiniteness.} For any vector \(v\), \(v^\top \Gamma_\phi v = \|B_\phi v\|_2^2 \geq 0\), so no eigenvalue clipping, softplus eigenvalue head, QR step, or Gram-Schmidt orthonormalization is required.
	\item \textbf{Rank control.} \(\operatorname{rank}(\Gamma_\phi) \leq r\), so the residual covariance is deliberately low rank.
	\item \textbf{Channel-aware variance.} The diagonal variance used by the NLL and consistency losses is
		\[
			\operatorname{diag}\Gamma_\phi(c,x)=\sum_{k=1}^r B_k(c,x)^2,
		\]
		with separate variance fields for each physical channel.
\end{enumerate}

\begin{theorem}[PSD Low-Rank Covariance]\label{thm:cov-approx}
	Let \(B:H\times[0,T]\to\mathbb{R}^{r\times n}\) be any learned factor on an \(n\)-dimensional discretization of the field. Then \(\Gamma=B^\top B\) is symmetric positive semidefinite and has rank at most \(r\). Conversely, every rank-\(r\) positive semidefinite covariance matrix admits such a factorization.
\end{theorem}

\begin{proof}
	Symmetry follows from \((B^\top B)^\top=B^\top B\). For positive semidefiniteness, \(v^\top B^\top Bv=\|Bv\|_2^2\geq0\). The rank bound follows from \(\operatorname{rank}(B^\top B)\leq\operatorname{rank}(B)\leq r\). Conversely, if \(\Gamma\succeq0\) and \(\operatorname{rank}(\Gamma)\leq r\), its spectral decomposition \(\Gamma=Q\Lambda Q^\top\) gives \(B=\Lambda^{1/2}Q^\top\), padded with zero rows if needed.
\end{proof}

The theorem is finite-dimensional by design because it describes the implemented architecture. In the continuum, the same idea corresponds to learning \(r\) square-integrable factor fields and forming their outer products. True Mercer eigenfunctions lie outside the claim, and the factor is identifiable only up to orthogonal rotations of its rows. The experiments report mean, variance, sampled \(W_2\), and diagnostic PSD behavior, leaving eigenfunction recovery aside.

\section{Auxiliary Separable and Mercer Priors}
\label{app:auxiliary-priors}

Beyond the generic FNO drift head, the MNO architecture can be generalized into a more explicit functional prior with a separable drift kernel and a Mercer-style covariance head whose eigenvalues and orthonormal eigenfunctions are represented directly. We derive the formulation here as a compatible extension of the same drift/residual idea, useful when one wants identifiable basis functions or a stronger hand-coded operator prior.

\subsection{Separable Drift Prior}

Let \(H=L^2(D)\). A separable drift head replaces the FNO drift map by
\[
	[\mathcal{A}_\eta(u_0,t)](x)
	= \sum_{i=1}^{N_K}
	\alpha_i(u_0,t)\,\varphi_i(x)
	\langle \psi_i,u_0\rangle_{L^2(D)},
\]
where \(\psi_i\) are sensing functions, \(\varphi_i\) are actuation functions, and \(\alpha_i\) are learned coefficients. If \(\alpha_i\) depend on an encoder of \(u_0\), the head is a nonlinear finite-rank prior. If \(\alpha_i=\alpha_i(t)\), it is a classical separable approximation of a time-dependent linear integral operator.

\begin{proposition}[Separable Drift Approximation]\label{prop:aux-separable}
	Let \(K_t:H\to H\) be a compact Hilbert-Schmidt operator with singular system \(\{(\sigma_i(t),\varphi_i^\star,\psi_i^\star)\}_{i\geq1}\). Its rank-\(N_K\) separable truncation
	\[
		[K_{t,N_K}u](x)=
		\sum_{i=1}^{N_K}\sigma_i(t)\varphi_i^\star(x)
		\langle \psi_i^\star,u\rangle_{L^2(D)}
	\]
	satisfies
	\[
		\|K_t-K_{t,N_K}\|_{\mathrm{HS}}^2
		= \sum_{i>N_K}\sigma_i(t)^2,
		\qquad
		\|K_t-K_{t,N_K}\|_{\mathrm{op}}
		= \sigma_{N_K+1}(t).
	\]
	Consequently, compact linear drift components with small singular tails admit accurate separable drift heads, and neural approximation of \(\varphi_i^\star,\psi_i^\star\), and \(\sigma_i(t)\) adds the usual function-approximation error.
\end{proposition}

\begin{proof}
	The proof is the Schmidt decomposition of a compact Hilbert-Schmidt operator. Orthogonality of the singular functions gives the Hilbert-Schmidt tail identity, and the operator-norm error of the truncated singular expansion is the first omitted singular value. Replacing the exact singular functions and coefficients by neural approximations perturbs the finite sum continuously in \(L^2(D)\).
\end{proof}

\subsection{Mercer Covariance Head Prior}

The analogous covariance prior writes the residual covariance as an explicit Mercer expansion
\[
	\Gamma_\zeta(u_0,t)
	= \sum_{k=1}^r
	\lambda_k(u_0,t)\,
	\psi_k(\cdot;u_0,t)\otimes\psi_k(\cdot;u_0,t),
\]
with \(\lambda_k(u_0,t)\geq0\) and \(\langle\psi_i,\psi_j\rangle_{L^2(D)}=\delta_{ij}\). A bounded positive eigenvalue head can be written, for example, as
\[
	\lambda_k(u_0,t)
	= C_\lambda\tanh\!\left(\frac{\exp(\widehat{\lambda}_k(u_0,t))}{C_\lambda}\right),
\]
and orthonormality can be imposed by applying a differentiable Gram-Schmidt step to raw fields \(\widetilde{\psi}_k\):
\[
	\psi_k
	=
	\frac{\widetilde{\psi}_k-\sum_{j<k}\langle \widetilde{\psi}_k,\psi_j\rangle\psi_j}
	{\left\|\widetilde{\psi}_k-\sum_{j<k}\langle \widetilde{\psi}_k,\psi_j\rangle\psi_j\right\|_{L^2(D)}}.
\]

\begin{proposition}[Mercer Covariance Approximation]\label{prop:aux-mercer}
	Let \(\Gamma^\star(u,t)\) be a positive semidefinite trace-class covariance operator on \(H\), with eigenpairs \(\{(\lambda_k^\star(u,t),\psi_k^\star(u,t))\}_{k\geq1}\). For fixed \((u,t)\), the rank-\(r\) Mercer truncation has Hilbert-Schmidt error
	\[
		\left\|\Gamma^\star(u,t)-\sum_{k=1}^r
		\lambda_k^\star(u,t)\psi_k^\star\otimes\psi_k^\star
		\right\|_{\mathrm{HS}}^2
		= \sum_{k>r}\lambda_k^\star(u,t)^2.
	\]
	If the learned head approximates the first \(r\) eigenvalues and normalized eigenfunctions with errors \(|\lambda_k-\lambda_k^\star|\leq\delta_\lambda\) and \(\|\psi_k-\psi_k^\star\|_{L^2}\leq\delta_\psi\), then
	\[
		\|\Gamma^\star-\Gamma_\zeta\|_{\mathrm{HS}}
		\leq
		\left(\sum_{k>r}\lambda_k^{\star 2}\right)^{1/2}
		+ r\delta_\lambda
		+ (2\delta_\psi+\delta_\psi^2)\sum_{k=1}^r\lambda_k^\star
		+ r\delta_\lambda(2\delta_\psi+\delta_\psi^2).
	\]
\end{proposition}

\begin{proof}
	The spectral theorem gives the Mercer expansion and the Hilbert-Schmidt tail identity. For the learned finite-rank part,
	\[
		\|\psi_k^\star\otimes\psi_k^\star-\psi_k\otimes\psi_k\|_{\mathrm{HS}}
		\leq 2\|\psi_k-\psi_k^\star\|_{L^2}+\|\psi_k-\psi_k^\star\|_{L^2}^2,
	\]
	using \(\|f\otimes g\|_{\mathrm{HS}}=\|f\|_{L^2}\|g\|_{L^2}\) and unit normalization. Summing the eigenvalue and eigenfunction perturbations and adding the spectral tail gives the bound.
\end{proof}

\subsection{Relation to the Implemented Factor Head}

The implemented covariance head is a relaxation of the Mercer prior. If a Mercer head is available, defining \(B_k(u_0,t)=\sqrt{\lambda_k(u_0,t)}\,\psi_k(\cdot;u_0,t)\) gives
\[
	B_\phi(u_0,t)^\top B_\phi(u_0,t)
	=
	\sum_{k=1}^r\lambda_k(u_0,t)\psi_k\otimes\psi_k.
\]
Conversely, a general learned factor \(B_\phi\) defines the same PSD rank-\(r\) covariance while leaving orthonormal eigenfunctions unidentified, since orthogonal rotations of the factor rows leave \(B_\phi^\top B_\phi\) unchanged. The main results use the factor head, and the separable/Mercer version above remains an optional structured prior.

\section{Training Details}
\label{app:training}

\textbf{Loss function.} The training objective used by the implementation is
\[
	\mathcal{L}
	= s_{\mathrm{NLL}}\mathcal{L}_{\mathrm{NLL}}
	+ \gamma\mathcal{L}_{\mathrm{consistency}}
	+ \epsilon\mathcal{L}_{\mathrm{martingale}}
	+ \delta\mathcal{L}_{\mathrm{reg}},
\]
where
\begin{itemize}
	\item \(\mathcal{L}_{\mathrm{NLL}}\) is the Gaussian negative log-likelihood of ensemble residuals using the predicted per-channel variance \(\sum_k B_k^2\), clamped to \([10^{-5},10^2]\).
	\item \(\mathcal{L}_{\mathrm{consistency}}\) is the relative \(L^2\) error between predicted analytical variance and empirical ensemble variance when multiple ensemble members are available.
	\item \(\mathcal{L}_{\mathrm{martingale}}\) is the squared ensemble mean of \(u_T-u_0-A_\theta(u_0,t)\), encouraging the residual to be centered.
	\item \(\mathcal{L}_{\mathrm{reg}}\) is a mild \(L^2\) regularization on the learned factor \(B_\phi\).
\end{itemize}
The default weights are \(s_{\mathrm{NLL}}=1.0\), \(\gamma=0.1\), \(\epsilon=0.1\), and \(\delta=0.01\). During early warmup, the NLL scale is ramped up so the mean head stabilizes before full variance learning.

\begin{table}[H]
	\centering
	\caption{Hyperparameters for MNO Case Studies (Standard Mode)}
	\label{tab:hyperparams}
	\begin{tabular}{lc}
		\toprule
		Parameter & Value \\
		\midrule
		Shared Backbone Width & 48 \\
		Shared Fourier Modes & 16 \\
		Backbone Layers & 4 \\
		Drift/Cov Head Depth & projection head / independent FNO \\
		Residual Factor Rank (\(r\)) & 16 \\
		Noise Type & Gaussian \\
		Optimizer & AdamW \\
		Learning Rate & \(10^{-3}\) \\
		Batch Size & 256 (1D), 32--64 (2D) \\
		Training Epochs & 120 \\
		Warmup Epochs (Mean-only) & 10 \\
		\bottomrule
	\end{tabular}
\end{table}

\textbf{Reproducibility details.} All reported result bundles, including the standard run, full run, and appendix diagnostic re-runs, were produced on a single AWS EC2 \texttt{g7e.2xl} instance using an NVIDIA RTX PRO 6000 Blackwell Server Edition GPU (sm\_120, CUDA 12.8, PyTorch 2.9.1).

\subsection{Loss Function Analysis: NLL and Moment Consistency}
MNO trains with NLL plus moment consistency, avoiding a direct Hilbert-Schmidt covariance loss.
\begin{itemize}
	\item \textbf{NLL term.} Ensemble members are treated as samples from the predicted Gaussian residual.
		\begin{equation}
			\mathcal{L}_{NLL}
			= \frac{1}{2}\left(\log v_\phi + \frac{(u_T-m_\theta)^2}{v_\phi}\right),
			\qquad v_\phi(c,x)=\sum_k B_k(c,x)^2,
		\end{equation}
		averaged over batch, ensemble, channels, and space.
	\item \textbf{Consistency term.} When \(N_{\mathrm{ens}}>1\), the predicted variance is compared to the empirical ensemble variance. This anchors the residual factor to observed marginal moments without claiming full covariance recovery.
	\item \textbf{Martingale centering.} The residual mean penalty enforces the terminal analogue of \(\mathbb{E}[M_t\mid u_0]=0\).
\end{itemize}

\section{Additional Experimental Results}
\label{app:additional-results}

This section reports benchmark results by task family. Each table lists the methods present in the corresponding result file, and ``--'' marks metrics absent from that run.

\subsection{Full 1D SPDE Results}
\label{app:1d-full}

Table~\ref{tab:1d-spde-full} reports the 1D stochastic Burgers and \(\phi^4\) terminal-law runs. The SPDEBench wrapper uses the same \(\phi^4\) data and gives the same values as the raw \(\phi^4\) run.

\begin{table}[H]
	\centering
	\caption{1D stochastic terminal-law benchmarks. Lower is better for all metrics.}
	\label{tab:1d-spde-full}
	\small
	\resizebox{\linewidth}{!}{%
		\begin{tabular}{llccc}
			\toprule
			Benchmark & Method & Mean RMSE & Var.\ RMSE & \(W_2\) \\
			\midrule
			Stochastic Burgers & Wiener-Chaos & 0.7746 & 0.0400 & 0.7271 \\
			Stochastic Burgers & SDENO-style & 0.7748 & 0.0409 & 0.7203 \\
			Stochastic Burgers & F-SPDENO-style & 0.7791 & 0.0372 & 0.7304 \\
			Stochastic Burgers & Neural SPDE & 0.6265 & 0.0378 & 0.6483 \\
			Stochastic Burgers & FNO & 0.0217 & -- & 0.0179 \\
			Stochastic Burgers & MNO & 0.0230 & 0.000438 & 0.00954 \\
			\midrule
			\(\phi^4\) / SPDEBench & Wiener-Chaos & 0.6420 & 0.0155 & 0.8111 \\
			\(\phi^4\) / SPDEBench & SDENO-style & 0.6422 & 0.0165 & 0.7990 \\
			\(\phi^4\) / SPDEBench & Neural SPDE & 0.6432 & 0.0359 & 0.6572 \\
			\(\phi^4\) / SPDEBench & MNO & 0.0086 & 0.000203 & 0.00547 \\
			\bottomrule
		\end{tabular}
	}
\end{table}

\subsection{1D Zero-Shot Superresolution}

Table~\ref{tab:superres-full} reports the full resolution sweep for the 1D superresolution run. MNO remains far below Neural SPDE at the zero-shot evaluation resolutions, while deterministic/stochastic-surrogate baselines are slightly lower on some mean and variance entries.

\begin{table}[H]
	\centering
	\caption{1D zero-shot superresolution. Models are trained at resolution 32 and evaluated at the listed resolution. Lower is better.}
	\label{tab:superres-full}
	\small
	\begin{tabular}{llcc}
		\toprule
		Resolution & Method & Mean RMSE & Var.\ RMSE \\
		\midrule
		32 & Wiener-Chaos & 0.0372 & 0.001322 \\
		32 & SDENO-style & 0.0385 & 0.001285 \\
		32 & Neural SPDE & 0.1023 & 0.002002 \\
		32 & MNO & 0.0390 & 0.001188 \\
		\midrule
		64 & Wiener-Chaos & 0.0188 & 0.001057 \\
		64 & SDENO-style & 0.0198 & 0.001061 \\
		64 & Neural SPDE & 0.1403 & 0.001322 \\
		64 & MNO & 0.0205 & 0.001261 \\
		\midrule
		128 & Wiener-Chaos & 0.0409 & 0.002675 \\
		128 & SDENO-style & 0.0397 & 0.002144 \\
		128 & Neural SPDE & 0.1957 & 0.001140 \\
		128 & MNO & 0.0407 & 0.002537 \\
		\bottomrule
	\end{tabular}
\end{table}

\subsection{Full Rough Volatility Results}
\label{app:rough-vol-full}

Table~\ref{tab:rough-vol-full} reports the full Hurst-parameter sweep averaged over five evaluation seeds. The metrics are terminal-marginal mean RMSE, variance RMSE, and empirical \(W_2\), with pathwise roughness preservation outside this evaluation.

\begin{table}[H]
	\centering
	\caption{Rough volatility terminal-marginal results across Hurst parameter \(H\). Lower is better for all metrics.}
	\label{tab:rough-vol-full}
	\small
	\resizebox{\linewidth}{!}{%
		\begin{tabular}{llccc}
			\toprule
			\(H\) & Method & Mean RMSE & Var.\ RMSE & \(W_2\) \\
			\midrule
			0.1 & Neural SDE & 0.0240 & 0.0990 & 0.0668 \\
			0.1 & Neural CDE & 0.0241 & 0.0994 & 0.0675 \\
			0.1 & MNO & 0.0230 & 0.0144 & 0.0257 \\
			\midrule
			0.2 & Neural SDE & 0.0239 & 0.0963 & 0.0666 \\
			0.2 & Neural CDE & 0.0239 & 0.0967 & 0.0674 \\
			0.2 & MNO & 0.0229 & 0.0137 & 0.0258 \\
			\midrule
			0.3 & Neural SDE & 0.0236 & 0.0944 & 0.0660 \\
			0.3 & Neural CDE & 0.0236 & 0.0949 & 0.0669 \\
			0.3 & MNO & 0.0226 & 0.0133 & 0.0254 \\
			\midrule
			0.4 & Neural SDE & 0.0234 & 0.0930 & 0.0655 \\
			0.4 & Neural CDE & 0.0234 & 0.0934 & 0.0665 \\
			0.4 & MNO & 0.0225 & 0.0129 & 0.0251 \\
			\midrule
			0.5 & Neural SDE & 0.0233 & 0.0920 & 0.0653 \\
			0.5 & Neural CDE & 0.0234 & 0.0924 & 0.0663 \\
			0.5 & MNO & 0.0223 & 0.0128 & 0.0249 \\
			\bottomrule
		\end{tabular}
	}
\end{table}

\subsection{Generative Efficiency}

Table~\ref{tab:generative-full} reports the NFE sweep for the conditional diffusion baseline. Both methods were trained under a 60-second wall-clock budget, and the reported timing run reached the 500-epoch MNO cap in 23.1 seconds and 21 diffusion epochs in 61.6 seconds because the diffusion target has higher per-epoch cost.

\begin{table}[H]
	\centering
	\caption{Generative-efficiency comparison on stochastic Burgers. Time is seconds per evaluation.}
	\label{tab:generative-full}
	\begin{tabular}{lccc}
		\toprule
		Method & NFE & Time (s) & \(W_2\) \\
		\midrule
		MNO & 1 & 0.00703 & 0.00813 \\
		Diffusion & 25 & 0.02134 & 0.15895 \\
		Diffusion & 50 & 0.04207 & 0.15698 \\
		Diffusion & 100 & 0.08386 & 0.16016 \\
		\bottomrule
	\end{tabular}
\end{table}

\subsection{Reaction-Diffusion Event Risk}

Table~\ref{tab:reaction-full} reports the event-risk run. The task uses a thresholded terminal event protocol in place of \(W_2\), where a realization is labeled if \(\max_x |u(T,x)| \geq 10\), and an initial condition is positive if at least 10\% of its ensemble realizations cross the threshold.

\begin{table}[H]
	\centering
	\caption{Reaction-diffusion event-risk results. Values are averaged over evaluation seeds.}
	\label{tab:reaction-full}
	\small
	\resizebox{\linewidth}{!}{%
		\begin{tabular}{lccccccc}
			\toprule
			Method & Mean RMSE & Var.-peak corr. & Risk score & Precision & Recall & F1 & Accuracy \\
			\midrule
			FNO & 0.1133 & 0.0000 & 0.7600 & 1.0000 & 0.9441 & 0.9712 & 0.9550 \\
			Neural SPDE & 3.4899 & -0.0237 & 0.6083 & 0.9835 & 0.8571 & 0.9159 & 0.8733 \\
			MNO (width 24) & 0.9948 & -0.3328 & 0.6453 & 1.0000 & 0.9193 & 0.9579 & 0.9350 \\
			MNO (width 48) & 1.0086 & -0.4307 & 0.6518 & 1.0000 & 0.8944 & 0.9443 & 0.9150 \\
			\bottomrule
		\end{tabular}
	}
\end{table}

\subsection{Full 2D Results}
\label{app:2d-full}

\begin{figure}[!htb]
	\centering
	\includegraphics[width=0.85\textwidth]{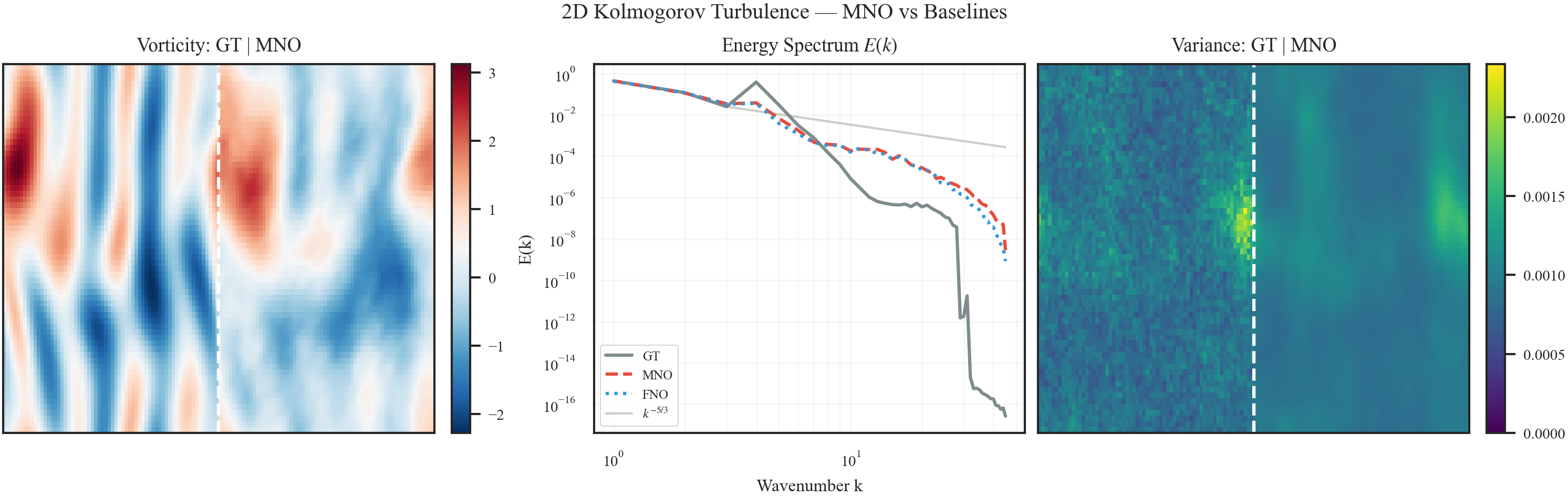}
	\caption{2D turbulent flow (Kolmogorov-forced, NS-style, \(64\times64\)). The figure shows the MNO mean prediction and its variance field, with numerical comparisons against all reported baselines given in Table~\ref{tab:2d-field-full}.}
	\label{fig:turbulent-flow-full}
\end{figure}

Table~\ref{tab:2d-field-full} reports all 2D field baselines. Only MNO emits a variance field in these runs.

\begin{table}[!htb]
	\centering
	\caption{2D field benchmarks. Lower is better. Variance RMSE is reported only for models that emit a variance field.}
	\label{tab:2d-field-full}
	\begin{tabular}{lcccc}
		\toprule
		Model & Gray-Scott mean & Gray-Scott var. & Turbulent mean & Turbulent var. \\
		\midrule
		FNO & 0.0054 & -- & 0.6583 & -- \\
		ResNet & 0.0198 & -- & 0.5806 & -- \\
		U-Net & 0.0236 & -- & 0.0706 & -- \\
		MNO & 0.0345 & 0.001441 & 0.6687 & 0.6896 \\
		\bottomrule
	\end{tabular}
\end{table}

\begin{figure}[!htb]
	\centering
	\includegraphics[width=0.85\textwidth]{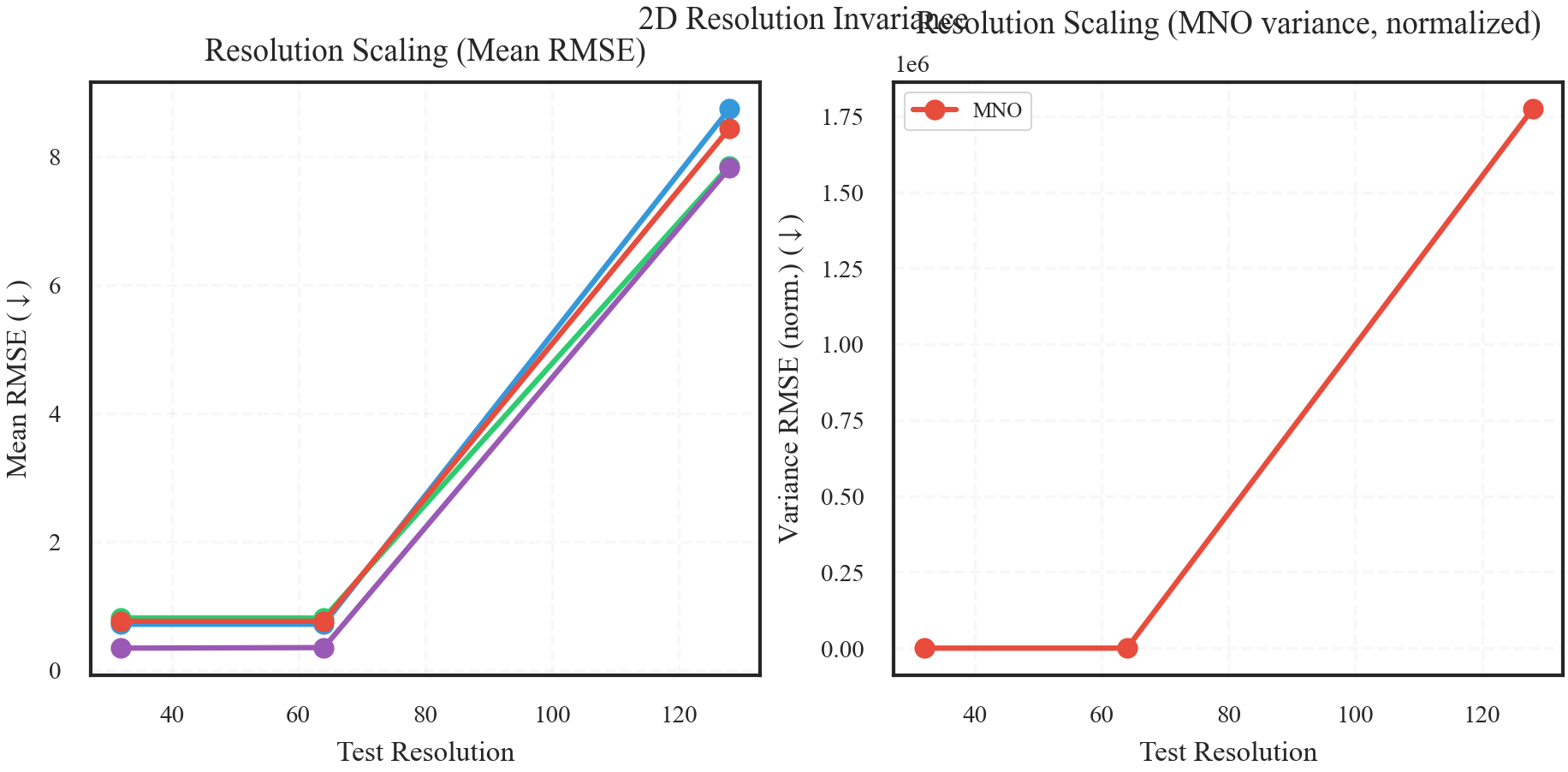}
	\caption{2D resolution transfer. Models are trained at \(64\times64\) and evaluated zero-shot at \(32\times32\), \(64\times64\), and \(128\times128\). All spectral and CNN baselines degrade sharply at \(128\times128\) because \texttt{modes=16} maps to different physical wavenumbers when the spatial grid changes, while MNO sits at parity with the FNO baseline at every tested resolution. Full mean-RMSE values are reported in Table~\ref{tab:2d-resolution-full}.}
	\label{fig:2d-resolution}
\end{figure}

\begin{table}[!htb]
	\centering
	\caption{2D resolution transfer. Mean RMSE is reported at each evaluation grid, and all models are trained at \(64\times64\). The result file also stores an MNO variance RMSE at \(64\times64\), while comparable variance metrics are absent for the other entries.}
	\label{tab:2d-resolution-full}
	\begin{tabular}{lcccc}
		\toprule
		Model & \(32\times32\) mean & \(64\times64\) mean & \(128\times128\) mean & \(64\times64\) var. \\
		\midrule
		FNO & 0.7125 & 0.7125 & 8.7421 & -- \\
		ResNet & 0.8110 & 0.8103 & 7.8570 & -- \\
		U-Net & 0.3408 & 0.3486 & 7.8214 & -- \\
		MNO & 0.7602 & 0.7602 & 8.4394 & 0.2510 \\
		\bottomrule
	\end{tabular}
\end{table}

The 2D Gray-Scott system exhibits sharp, spatially coherent Turing-pattern structures. In the reported run FNO has lower mean RMSE than MNO (Table~\ref{tab:2d-field-full}), because the training configuration of MNO adds a residual factor for stochastic terminal uncertainty while this regime is dominated by mean-field pattern prediction under the tested noise level.

\begin{figure}[!htb]
	\centering
	\includegraphics[width=0.85\textwidth]{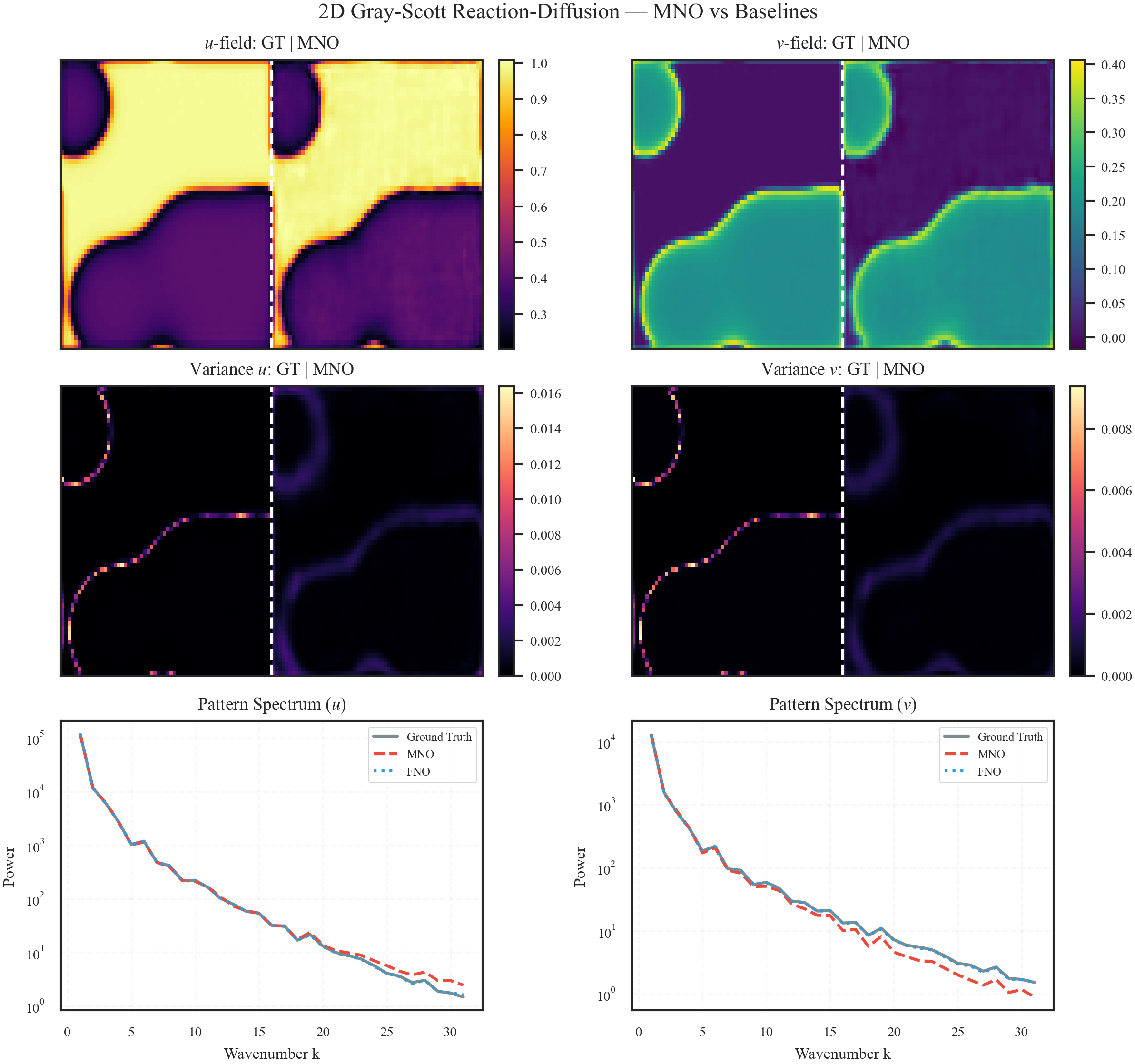}
	\caption{2D Gray-Scott reaction-diffusion. The reported run gives mean RMSE \(0.0054\) for FNO and \(0.0345\) for MNO, with all reported 2D field metrics listed in Table~\ref{tab:2d-field-full}.}
	\label{fig:gray-scott}
\end{figure}

\subsection{Burgers Ablations: Rank, Backbone, Loss, Factor Dimension}
\label{app:burgers-ablations}

We isolate four architectural and loss choices on the stochastic Burgers benchmark. All cells share \(n_{\mathrm{train}}=1500\), \(n_{\mathrm{ensemble}}=256\), \(n_{\mathrm{epochs}}=150\), and 5 evaluation seeds, with only the swept axis changing.

\paragraph{Rank.} Sweeping the residual rank \(r\) with shared backbone and the full loss, mean RMSE decreases monotonically across the grid:
\begin{table}[H]
	\centering
	\caption{Rank ablation on stochastic Burgers (5-seed mean RMSE, lower is better).}
	\label{tab:burgers-rank-ablation}
	\begin{tabular}{lccccccc}
		\toprule
		\(r\) & 1 & 2 & 4 & 8 & 16 & 32 & 64 \\
		\midrule
		mean RMSE & 0.0237 & 0.0230 & 0.0229 & 0.0224 & 0.0204 & 0.0193 & 0.0181 \\
		\bottomrule
	\end{tabular}
\end{table}
Returns are diminishing past \(r=16\), and the gap from \(r=16\) to \(r=64\) is roughly \(11\%\).

\paragraph{Backbone sharing.} A single shared FNO backbone with separate projection heads outperforms two independent backbones (mean RMSE \(0.0204\) shared vs.\ \(0.0280\) split), at half the parameter count. We attribute the gap to reduced gradient interference between drift and basis updates when the lower layers see a single optimization signal.

\paragraph{Loss leave-one-out.} Removing each auxiliary term from the full objective:
\begin{table}[H]
	\centering
	\caption{Loss leave-one-out (5-seed mean RMSE).}
	\label{tab:burgers-loss-ablation}
	\begin{tabular}{lccccc}
		\toprule
		Variant & Full & No consistency & No martingale & No reg & NLL only \\
		\midrule
		Mean RMSE & 0.0204 & 0.0259 & 0.0206 & 0.0206 & 0.0259 \\
		\bottomrule
	\end{tabular}
\end{table}
The variance-consistency term (\(\gamma\)) is the load-bearing auxiliary signal, since removing it alone produces the same degradation as removing every auxiliary (\emph{NLL only}). The martingale and regularization terms have negligible marginal effect at this benchmark scale, but we retain them because of the residual-centering diagnostic in Appendix~\ref{app:martingale}.

\paragraph{Factor dimension.} Sweeping the number of factor channels over the same grid as the rank sweep produces identical mean-RMSE values within numerical noise, because factor channels and rank are the same architectural axis at single-channel resolution and differ only in their cost profile. Sample time and peak GPU memory scale linearly with the channel count, while RMSE plateaus past \(r=16\) per the rank sweep.

\subsection{Cross-Covariance Fidelity}
\label{app:crosscov}

This experiment isolates a known limitation of the diagonal-only consistency term, since \(B_\phi\) is directly optimized to match the marginal variance \(\operatorname{diag}(B_\phi^\top B_\phi)\). The full covariance \(B_\phi^\top B_\phi\) is constrained only indirectly, through the ensemble negative log-likelihood and the rank.

We train MNO and a Neural-SPDE baseline on stochastic Burgers at \(n_x = 32\) with a large reference ensemble (\(n_{\mathrm{ensemble}} = 1024\), \(n_{\mathrm{train}} = 400\), 5 seeds) and report the relative Frobenius distance
\[
	\texttt{cov\_frobenius\_rel} \;=\; \frac{\|\widehat{C} - C_{\mathrm{emp}}\|_F}{\|C_{\mathrm{emp}}\|_F},
\]
where \(\widehat{C}\) is \(B_\phi^\top B_\phi\) for MNO and the sample covariance of generated trajectories for Neural SPDE.

\begin{table}[H]
	\centering
	\caption{Cross-covariance fidelity (5-seed mean ± 95\% CI). Lower is better.}
	\label{tab:crosscov}
	\begin{tabular}{lcc}
		\toprule
		Method & \texttt{cov\_frobenius\_rel} & mean RMSE \\
		\midrule
		MNO & \(1.22 \pm 0.05\) & \(0.049 \pm 0.001\) \\
		Neural SPDE & \(0.78 \pm 0.02\) & \(0.107 \pm 0.005\) \\
		\bottomrule
	\end{tabular}
\end{table}

Two observations follow. First, MNO retains a clear advantage in mean RMSE (\(2.2\times\) lower), consistent with the headline benchmarks. Second, MNO is worse on full-covariance recovery, since a sample-based reconstruction from a generative model better matches the empirical cross-covariance than the directly factorized \(B_\phi^\top B_\phi\) does, despite \(B_\phi\) recovering marginal variance accurately (Appendix~\ref{app:burgers-ablations}, rank ablation).

This diagnostic targets the loss design. The consistency loss matches \(\sum_k B_k(c,x)^2 \approx \mathrm{Var}[u_t(c,x)]\) per spatial point and leaves off-diagonal covariance errors unpenalized. A Hilbert--Schmidt cross-covariance term, or a Hutchinson-style random-projection estimator \(\mathbb{E}_v\big[(v^\top B_\phi^\top B_\phi v - v^\top C_{\mathrm{emp}} v)^2\big]\), would close this gap without changing the model. We retain the diagonal-only formulation for tractability and report this experiment as a known limitation rather than a result claim.

\subsection{Coverage Calibration}
\label{app:coverage}

The repository includes a coverage script (\texttt{experiments/compute\_coverage.py}), but the standard result bundle lacks \texttt{coverage\_calibration.json}. Coverage numbers are absent here. The reported calibration evidence is limited to variance RMSE, sampled \(W_2\), and the uncertainty-decomposition diagnostic in Appendix~\ref{app:verification-suite}.

\subsection{SPDEBench Verification and Experimental Scope}

The \(\phi^4\) field theory benchmark (Table~\ref{tab:application-results}) also constitutes the SPDEBench evaluation, since both use the same underlying data and simulator. We ran the benchmark through two independent packaging pipelines, the raw simulator and the SPDEBench wrapper, obtaining identical MNO \(W_2 = 0.00547\) in both cases and confirming reproducibility. Table~\ref{tab:application-results} consolidates these into a single row labeled ``\(\phi^4\) / SPDEBench.''

\section{Computational Complexity}
\label{app:complexity}

MNO is designed for one-shot inference, so the cost depends on evaluating the neural operator heads at the requested grid, with no \(N_t\)-step solver rollout. The dominant cost is the FNO backbone, while moment extraction from the residual factor is linear in the number of grid points once the factor has been produced.

\begin{table}[t]
	\centering
	\caption{Complexity Comparison (Inference per sample generated)}
	\label{tab:complexity}
	\resizebox{\linewidth}{!}{%
		\begin{tabular}{lccc}
			\toprule
			Method & Time Complexity & Space Complexity & Grid Invariant? \\
			\midrule
			Standard SDE Solver (Euler-Maruyama) & \(O(N_t \cdot N_x)\) & \(O(N_x)\) & No \\
			Monte Carlo Ensemble (\(M\) paths) & \(O(M \cdot N_t \cdot N_x)\) & \(O(M \cdot N_x)\) & No \\
			Score-Based Diffusion & \(O(N_{\text{steps}} \cdot N_x)\) & \(O(N_x)\) & Yes \\
			Spectral Methods (FFT) & \(O(N_t \cdot N_x \log N_x)\) & \(O(N_x)\) & No \\
			\textbf{MNO (moments)} & \(\mathbf{O(\mathrm{FNO}(N_x)+rCN_x)}\) & \(\mathbf{O(rCN_x)}\) & \textbf{Yes} \\
			\textbf{MNO (S samples)} & \(\mathbf{O(\mathrm{FNO}(N_x)+SrCN_x)}\) & \(\mathbf{O((S+r)CN_x)}\) & \textbf{Yes} \\
			\bottomrule
		\end{tabular}
	}
\end{table}

\textbf{Moment Evaluation.}
After the FNO head emits \(B_\phi\), the variance field is \(\sum_k B_k^2\), which costs \(O(rCN_x)\) in 1D and \(O(rCN_xN_y)\) in 2D. Sampling \(S\) residual fields multiplies Gaussian coefficients by the flattened factor and costs \(O(SrCN_x)\) in 1D. The implementation avoids dense \(O(N_x^2)\) covariance matrices while preserving an explicit PSD covariance factor.

\section{Resolution Invariance and Discretization Scope}
\label{app:resolution}

A stronger \(O(h^2)\) convergence guarantee would require explicit separable kernels and Gram-Schmidt Mercer eigenfunctions, which lie outside the implemented model. The claim here is narrower, namely that the FNO drift head and low-rank residual factor can be evaluated on new grids, with diagnostics measuring whether mean and variance predictions remain stable under that transfer.

\begin{theorem}[Resolution Transfer Scope]\label{thm:resolution-invariance}
	For a fixed set of learned FNO weights, MNO can be evaluated on any grid supported by the Fourier layers and normalization pipeline. The residual factor \(B_\phi\) is emitted on that grid, so moment computation avoids grid-specific dense covariance matrices. This theorem is an architectural statement, and zero-shot accuracy must be measured for each data regime.
\end{theorem}

\begin{proof}
	The FNO layers operate by truncating and multiplying Fourier modes on the supplied grid, allowing the same learned weights to be applied at different resolutions. The MNO residual head emits \(rC\) channels on that grid and reshapes them into \(B_\phi\). Since the variance is computed by \(\sum_k B_k^2\), the computation introduces neither a grid-specific covariance matrix nor an eigendecomposition. These facts establish evaluability across grids, while an a priori error rate for a trained model would require additional assumptions, so the paper reports explicit transfer diagnostics.
\end{proof}

\begin{remark}
	Resolution transfer is a critical property that distinguishes neural operators from grid-fixed surrogates. In the MNO implementation it means the following.
	\begin{itemize}
		\item the same FNO weights can be evaluated on different grid sizes.
		\item the residual factor is produced as a field on that grid and avoids a dense covariance matrix.
		\item zero-shot transfer must still be measured, because finite Fourier modes, normalization, and data distribution can all limit transfer.
	\end{itemize}
\end{remark}

\section{Temporal Consistency and Semigroup Property}
\label{app:temporal-consistency}

MNO is trained as a one-shot terminal marginal operator, with calibrated temporal consistency outside the present objective. The distinction matters because composing Gaussian marginal predictions is exact only under additional linearity or state-independent covariance assumptions. For nonlinear learned heads, evaluating \(\mathcal{A}_\theta\) or \(B_\phi\) at the mean of an intermediate Gaussian is an approximation and falls short of proving equality in distribution.

The Martingale Mirror experiment supplies the practical diagnostic by asking what happens when the one-shot model is reused autoregressively. Its failure to preserve the rough Hurst exponent fits the paper's scope, because MNO learns terminal marginals efficiently, while pathwise temporal consistency requires a different training objective or a history-dependent architecture.

\section{Residual Centering and Martingale Scope}
\label{app:martingale}

A key design principle of MNO is to preserve the terminal consequences of a martingale residual, namely zero initial noise, zero conditional mean, and positive semidefinite covariance. The full filtration-level tower property \(\mathbb{E}[M_t\mid\mathcal{F}_s]=M_s\) for all \(s<t\) remains outside the enforced architecture and would require pathwise conditioning or explicit temporal consistency constraints.

\begin{theorem}[Terminal Residual Centering]\label{thm:martingale}
	Let MNO generate
	\[
		u_t=u_0+\mathcal{A}_\theta(u_0,t)+B_\phi(u_0,t)^\top\xi,\qquad \mathbb{E}[\xi]=0,\quad \operatorname{Cov}(\xi)=I_r,
	\]
	with \(\xi\) independent of \(u_0\).
	Then
	\[
		\mathbb{E}[u_t\mid u_0]=u_0+\mathcal{A}_\theta(u_0,t),
		\qquad
		\operatorname{Cov}(u_t\mid u_0)=B_\phi(u_0,t)^\top B_\phi(u_0,t).
	\]
	If \(t=0\), the temporal gate gives \(\mathcal{A}_\theta(u_0,0)=0\) and \(B_\phi(u_0,0)=0\), hence \(u_t=u_0\).
\end{theorem}

\begin{proof}
	Linearity of expectation gives \(\mathbb{E}[B_\phi^\top\xi\mid u_0]=B_\phi^\top\mathbb{E}[\xi]=0\). The covariance identity follows from
	\[
		\operatorname{Cov}(B_\phi^\top\xi\mid u_0)=B_\phi^\top\operatorname{Cov}(\xi)B_\phi=B_\phi^\top B_\phi.
	\]
	The gate statement follows by substituting \(g(0)=1-\exp(0)=0\).
\end{proof}

No Gaussian assumption is used in this theorem, and Gaussian coefficients are only the experimental instantiation used for the NLL and sampling results.

\begin{remark}
	This theorem formalizes the implemented guarantee. The experiments then ask whether the learned factor is useful for stochastic operator learning.
	\begin{itemize}
		\item \textbf{Uncertainty quantification.} Sampling from \(\mathcal{G}_{\theta,\phi}\) yields realizations with learned marginal statistics.
		\item \textbf{Risk assessment.} The variance field gives a directly usable uncertainty envelope for Gaussian moment-based risk.
		\item \textbf{Scope.} Full pathwise martingale preservation and tail-risk calibration require additional constraints beyond this architecture.
	\end{itemize}
\end{remark}


\end{document}